\documentclass[twoside]{article}

\usepackage[accepted]{aistats2025}

\usepackage{hyperref}       
\usepackage{url}            
\usepackage{booktabs}       
\usepackage{amsfonts}       
\usepackage{nicefrac}       
\usepackage{xcolor}         
\usepackage{amsmath}
\usepackage{amsmath}
\usepackage{amssymb}
\usepackage{amsthm}
\usepackage{hyperref}
\usepackage{enumerate}
\usepackage{enumitem}
\usepackage{blindtext}

\usepackage[round]{natbib}

\usepackage{caption}
\usepackage{subcaption}
\usepackage{bbm}
\usepackage{graphicx}
\usepackage{color,soul}
\usepackage{float}
\usepackage{mathbbol}
\usepackage{algorithm}
\usepackage{mathdots}
\usepackage{xcolor}
\usepackage{blindtext}
\usepackage{float}
\usepackage{verbatim}
\usepackage[noend]{algpseudocode}
\usepackage{scalerel}
\makeatletter
\newcommand{\newalgname}[1]{%
  \renewcommand{\ALG@name}{#1}%
}
\usepackage{tikz}
\usetikzlibrary{patterns}
\usetikzlibrary{svg.path}
\usetikzlibrary{backgrounds,automata}
\usepackage{array}
\newtheorem{theorem}{Theorem}
\newtheorem{proposition}[theorem]{Proposition}
\newtheorem{lemma}[theorem]{Lemma}

\theoremstyle{definition}

\newtheorem{Assumption}{Assumption}

\newcommand{\Ebb}{\mathbb{E}}
\newcommand{\Nbb}{\mathbb{N}}
\newcommand{\Pbb}{\mathbb{P}}

\newcommand{\Rbb}{\mathbb{R}}

\newcommand{\Hbf}{\mathbf{H}}
\newcommand{\Kbf}{\mathbf{K}}

\newcommand{\Ecal}{\mathcal{E}}
\newcommand{\Fcal}{\mathcal{F}}

\newcommand{\Hcal}{\mathcal{H}}
\newcommand{\Mcal}{\mathcal{M}}

\newcommand{\Pcal}{\mathcal{P}}

\newcommand{\Xcal}{\mathcal{X}}

\newcommand{\one}{{\mathbf{1}}}

\renewcommand{\geq}{\geqslant}
\renewcommand{\leq}{\leqslant}

\newcommand{\dommes}{\mu}

\DeclareMathOperator{\argmin}{\text{arg min}}
\DeclareMathOperator{\argmax}{\text{arg max}}

\DeclareMathOperator{\softmax}{\text{softmax}}
\DeclareMathOperator{\pen}{\text{pen}}
\DeclareMathOperator{\DKL}{D_{KL}}
\DeclareMathOperator{\DHL}{H}

\makeatletter
\newcommand*\bigcdot{\mathpalette\bigcdot@{.5}}
\newcommand*\bigcdot@[2]{\mathbin{\vcenter{\hbox{\scalebox{#2}{$\m@th#1\bullet$}}}}}
\makeatother

\begin{document}

\runningauthor{Julien Aubert, Louis Köhler, Luc Lehéricy,  Giulia Mezzadri, Patricia Reynaud-Bouret}

\twocolumn[

\aistatstitle{
Model selection for behavioral learning data and applications to contextual bandits}

\aistatsauthor{ 
    Julien Aubert\textsuperscript{1} \And 
    Louis Köhler\textsuperscript{2} \And 
    Luc Lehéricy\textsuperscript{1} 
}

\aistatsauthor{    Giulia Mezzadri\textsuperscript{3} \And 
    Patricia Reynaud-Bouret\textsuperscript{1} }

\aistatsaddress{\textsuperscript{1}Université Côte d'Azur, CNRS, LJAD, France\\\textsuperscript{2}Université Côte d'Azur, LJAD, France
\\ \textsuperscript{3}Columbia University, Cognition and Decision Lab, United States}]

\begin{abstract}

Learning for animals or humans is the process that leads to behaviors better adapted to the environment. This process highly depends on the individual that learns and is usually observed only through the individual's actions. This article presents ways to use this individual behavioral data to find the model that best explains how the individual learns. We propose two model selection methods: a general hold-out procedure and an AIC-type criterion, both adapted to non-stationary dependent data. We provide theoretical error bounds for these methods that are close to those of the standard i.i.d. case. To compare these approaches, we apply them to contextual bandit models and illustrate their use on both synthetic and experimental learning data in a human categorization task.

\end{abstract}

\section{INTRODUCTION}

\subsection{Modeling Learning}

From a behavioral perspective, learning can be defined as ``a process by which an organism benefits from experience so that its future behavior is better adapted to its environment'' \citep{rescorla1988behavioral}. What if we want to model the learning process of an individual based only on the observations of its actions? That is, given an experiment of a learning task and some behavioral data of an individual performing that task, we want to find a model that explains the behavior and actions of that individual.  
This question lies in the computational modeling of behavioral data (also referred to as cognitive modeling) \citep{wilson2019ten,farrell2018computational,collins2022advances,Mezzadri2020,busemeyer2010cognitive}.



By its very nature, learning is a non stationary and dependent process.
An extensive theory on statistical estimation and model selection exists for independent and stationary data (see Section~\ref{section_relatedwork}). 
On the other hand, there are very few theoretical studies of statistical methods for non stationary and non independent data: in \citep{aubert23}, the properties of the MLE are studied for the \texttt{Exp3} model \citep{auer2002nonstochastic} on learning data; in \citep{aubert:generalpmle}, a very general model selection procedure is presented that can be applied to non stationary data but works with restrictive assumptions on the models.

In this work, we provide a new oracle inequality for a hold-out procedure specifically designed for non-stationary data that offers flexibility for a wide range of non-stationary learning scenarios. However this method does not encompass parameter estimation in each model. Hence we also provide a AIC-like penalized log-likelihood estimation. The corresponding theoretical result is an oracle inequality derived by applying  \citep{aubert:generalpmle}. It requires assumptions on the parametrization on the models that hold in particular for contextual bandits. 

\subsection{Contextual Bandits as Models of Learning}\label{section_contextualbandits}


A contextual bandit algorithm \citep{auer2002nonstochastic,lattimore_szepesvári_2020} is a decision-making framework where, at each time step, the learner observes contextual information, selects an action and receives feedback in the form of a reward based on both the chosen action and the context. Unlike the standard multi-armed bandit problem, contextual bandits incorporate additional information (the context) from the environment to guide action selection, allowing the model to adapt its policy dynamically based on the current situation. The usual goal of a reinforcement algorithm is to find an optimal policy that maximizes cumulative rewards. Contextual bandits have broad applications across machine learning, including recommendation systems, healthcare decision-making, and personalized medicine \citep{bouneffouf2019survey}. 

In this work, we use contextual bandits as models of learning. They are gaining popularity in cognitive psychology \citep{lan2016contextual,schulz2018putting} because they provide a simple yet effective framework for understanding how individuals adjust their decision 
based on past experiences and present contextual cues. Many traditional cognitive models, such as the Component Cue 
\citep{gluck1988evaluating} or the ALCOVE models \citep{Kruschke1992}, address similar problems and can be reinterpreted as contextual bandits.

One key advantage in using contextual bandits as models of learning is their flexibility—the ``context" can be almost anything (see Section~\ref{section_numericalillustrations} for an example), making the framework highly adaptable to a wide range of learning situations. Additionally, their ability to efficiently balance exploration and exploitation is critical in uncertain environments. Their relatively loose modeling approach allows for tractable representations of complex behavioral data which makes them easier to implement than traditional cognitive models.


\subsection{Contributions}

In what follows, $n$ is the number of observed choices and actions during the learning experiment.

    $\bullet$ In Section~\ref{section_holdout}, we show that for any finite family of models $\Mcal$, a hold-out estimator satisfies an oracle inequality with an $\mathcal{O}\left((\log n+\log |\Mcal|)/n\right)$ error bound, regardless of the nature of the models of learning.
    
    $\bullet$ In Section~\ref{section_penalizedmle}, we consider an AIC-type criterion 
    with a possibly infinite countable number of parametric models built from contextual bandits algorithms. We show an oracle inequality with an $\mathcal{O}\left(\log(n) /n\right)$ error bound 
    and explain the pros and cons of the AIC-type criterion vs the hold-out procedure.
    
    $\bullet$ Section~\ref{section_numericalillustrations} is devoted to numerical illustrations of both methods on both synthetic and experimental learning data in a categorization task. 
    Here the models in competition are 
    contextual bandits (see Appendix~\ref{section_codeandatadescription} for more details).
    
    $\bullet$ We present examples of bandits algorithms for which assumptions of Section~\ref{section_penalizedmle} are satisfied in Appendix~\ref{section_examplebandits}.
    
    $\bullet$ In Appendix~\ref{section_metalearning}, we discuss how bandits with expert advice can be used to model metalearning \citep{binz2023meta}, which refers to the processes by which an individual acquires knowledge about its own learning abilities, strategies, and preferences. We provide the details to obtain model selection methods and guarantees for metalearning as well.
    
    $\bullet$ The complete proofs of the theoretical results are given in Appendix~\ref{section_proofs}.

\subsection{Related Work}\label{section_relatedwork}

Hold-out estimators are commonly applied in cognitive modeling for learning data \citep{mezzadri2022hold,james2023strategy} for arbitrating between models. Theoretical analysis of hold-out procedures in the literature generally assumes the data stationary and independent \citep{massart2007concentration,arlot_lerasle,arlot2010survey}. There exist some limited results for time-dependent data \citep{Opsomer}, but these involve very different approaches from ours. A key challenge in our context is that the training set is not independent from the validation set, which prevents the use of 
techniques such as $V$-fold cross-validation.

Section~\ref{section_penalizedmle} is similar in design to the framework of  \citet{castellan} for selecting the best histogram for density estimation or more generally to non asymptotic model selection \citep{massart2007concentration}. The main difference is that we are in a non stationary and non independent framework. 

In Section~\ref{section_penalizedmle}, unlike standard approaches in reinforcement learning (RL) \citep{BubeckBianchi}, our goal is not to improve regret or to develop algorithms optimizing rewards as in \citep{dimakopoulou2017estimation,foster2019model,NEURIPS2020_751d5152}; instead, we aim to understand how an individual learns. We select the contextual bandit model that best fits an individual's learning curve from their learning data, without assuming the individual understands the context-action relationship. Thus, we seek the most realistic model rather than an optimal one. This theoretical statistical problem was first studied in \citep{aubert23}.
Our method goes further and accommodates misspecified models.

At first sight, this problem may seem similar to the perspectives of Imitation Learning (IL) and Inverse Reinforcement Learning (IRL).
Our work consists in trying to reproduce the learning curve of an expert (the individual under observation).
However, IL methods use data from an expert who has already mastered the task \citep{hussein2017imitation} and learn to perform it by copying them, while IRL \citep{arora2021survey} aims to infer an underlying reward function based on the expert's observed behavior across multiple trajectories. In contrast, in the field of cognition, experimenters control the reward function and seek to infer an individual's behavior based on a single learning trajectory.


Our work is close to \citet{huyuk2022inverse,schulz2015learning,schulz2018putting} who estimate how a learner's behavior evolves with Bayesian models, however they focus on reward estimation. Therefore we cannot compare their method (whose output is a reward function) with our method (whose output is a policy).

\section{FRAMEWORK AND NOTATIONS}\label{sec:frame}

The methodology we use is similar to the one in~\citep{wilson2019ten,farrell2018computational,daw2011trial}. It starts by considering a family of models that are relevant for the experiment and observed behavior: each model provides a family of candidate distributions for the sequence of choices made by the individual.
For each model, model fitting is usually done by maximum likelihood estimation (MLE). Finally, the best cognitive model is selected by hold-out or an Akaike-type information criterion (AIC). In the particular case of the hold-out procedure, we focus on the selection step, disregarding how the fitted models are obtained so that our results hold regardless of the method used to extract one distribution per model.

\subsection{Notations}

Given two integers $s \leq t$ and a sequence $(a_s)_{s \in \mathbb{Z}}$, we define $a_s^t = (a_s, \dots, a_t)$, with $a_s^t$ being the empty sequence when $s > t$. Let $\mathbb{N}^*$ denote the set of positive integers, and for any $ n \in \mathbb{N}^* $, we write $[n] = \{1, \ldots, n\}$. Lastly, we denote the natural logarithm by $\log$.

Let $ n \geq 3 $ be an integer. We observe the sequence of actions $(A_t)_{1 \leq t \leq n}$ defined on a Polish measure space $(\Omega, \Fcal, \mu)$ and adapted to a filtration $(\Fcal_t)_{1 \leq t \leq n}$. We denote the corresponding probability measure by $\mathbb{P}$ and the expectation by $\Ebb$.

Our goal is to estimate the successive conditional densities of $A_t$ w.r.t. $\mu$. If $\Omega$ is discrete and $\mu$ is the counting measure, these densities are given by the sequence
\begin{equation*}
p^\star_t(.) = \mathbb{P}(A_t = . \mid \Fcal_{t-1}), \quad \forall t \in [n].
\end{equation*}
In the following discussion, we will assume that $\mu$ is a fixed measure on $\Omega$. For general measured spaces $(\Omega, \Fcal, \mu)$, we assume that the conditional density of $A_t$ given $\Fcal_{t-1}$ with respect to $\mu$ exists, denoting it by $p^\star_t(.)$. Consequently, for all $x \in \Omega$, the sequence $(p^\star_t(x))_{1 \leq t \leq n}$ is predictable with respect to the filtration. Let $p^\star = (p^\star_t)_{t \in [n]}$ represent the vector of all successive conditional densities.

\subsection{Some Examples of Filtrations}
\label{sec:some_examples}

\paragraph{The Filtration Depends Only on Past Observations.}
A simple scenario consists in assuming that an action at time $t$ depends only on the history of past actions. In this case, the filtration is defined by $\Fcal_t = \sigma(A_1^t)$ for $t \geq 1$, and $\Fcal_0$ is the trivial $\sigma$-algebra with no prior information. Under this setup, the true density at time $t$, $p_t^\star$, is the conditional density of $A_t$ given the past actions $A_1^{t-1}$. For $t = 1$, $p_1^\star$ is the deterministic density of the first action since no prior actions have been taken.

This can occur as soon as the environment of the individual is fixed (see for instance \citep{rescorla1988behavioral,barron2015embracing}). Cognitive tasks in such situations usually look like a standard bandit problem, where the individual tries to optimize its reward by pulling arms. Such tasks have been used in humans for studies on addiction \citep{DBLP:journals/corr/BouneffoufRC17} or in rodents with the Skinner box \citep{skinner2019behavior}.


A classical model for such a task is the \texttt{Gradient Bandit} algorithm \citep{sutton2018reinforcement,mei2024stochastic}, with learning rate $\theta$. In this algorithm~\ref{gradientbandits_true}, the agent chooses an action $A_t$ among a finite set of actions $[K]$ where $K$ is a positive integer, and receives a reward $g_{A_t,t} > 0$ from the environment drawn from some probability distribution. The algorithm then uses policy gradient methods to directly optimize the probabilities of selecting each action.

\begin{algorithm}[H]
\caption{\texttt{Gradient Bandit} \citep{mei2023stochastic}}\label{gradientbandits_true}
\begin{algorithmic}
\State \textbf{Inputs: } $n$, $\theta > 0$, $K \in \Nbb^*$.
\State \textbf{Initialization: } $p_{\theta,1}=  \left(\frac{1}{K}, \ldots, \frac{1}{K}\right)$.
\For {$t \in [n]$}
\State Draw an action $A_t \sim p_t$ and receive a reward $g_{A_t,t} >0$.
\State Update for all $a \in [K]$,
\begin{equation*}
    p_{\theta,t}(a) = \frac{\exp\left(-\theta \sum_{s \in [t]} \hat{g}_{a,s}^{\theta}\right)}{\sum_{b \in [K]} \exp\left(-\theta \sum_{s \in [t]} \hat{g}_{b,s}^{\theta}\right)}
\end{equation*}
where $\hat{g}_{b,s}^{\theta} =\left(\one_{A_s = b}-p_{\theta,s}(b)\right) g_{A_s,s}.$

\EndFor
\end{algorithmic}
\end{algorithm}

A model $(p_\theta)_{\theta\in\Theta}$ is said to be {\it well-specified} if there exist $\theta^*$ such that the true distribution $p^*=p_{\theta^*}$. Properties of the MLE for a well-specified model in this particular learning framework have been explored in \citep{aubert23}.

\paragraph{The Filtration Depends on Past Observations and Additional Variables.} A more comprehensive approach to modeling learning \citep{wilson2019ten,marr2010vision,collins2022advances} involves the assumption that a learner's decisions are influenced not only by past choices but also by observable contextual variables. Before making a decision at time $t$, the individual has access to context information $X_t$. The natural filtration in this framework is defined as $\Fcal_{t-1}=\sigma(A_1^{t-1}, X_1^t)$, which captures both the learner's previous actions $A_1^{t-1}$ and all context variables $X_1^t$ available up to time $t$.

\subsection{Partial Log-likelihood} 

Given a sequence of distributions $p = (p_t)_{t\in [n]}$, we use as criterion the partial log-likelihood
\begin{equation}\label{eq_loglikelihood}
\ell_n(p) = \sum_{t=N}^{n} \log p_t(A_t)
\end{equation}
where the sum starts at $N=1$ in the AIC-type framework, and some $N>1$ for the hold-out procedure (in which the first $N-1$ actions are used to estimate the distribution in each model and the last $n-N+1$ to select a model). For the example of Section~\ref{sec:some_examples} where $\Fcal_t=\sigma(A_1^t)$ for $t\geq 1$ and $\Fcal_0$ is the trivial $\sigma$-algebra, $\ell_n(p)$ is exactly the log-likelihood $\log p(A_1^n)$. For the second example of Section~\ref{sec:some_examples}, we use the term ``partial'' log-likelihood as per \cite{cox1975partial}, because we are only interested in modeling the action process $A_1^n$ and not the entire vector of observations $(X_1,A_1,\ldots,X_n,A_n)$.

\subsection{Stochastic Risk Function}

Classical approaches \citep{massart2007concentration, spok2012,spok2017} typically define risk using an expectation of the contrast. For example, in i.i.d. scenarios, the log-likelihood is inherently linked to the Kullback-Leibler (KL) divergence between the estimated and the true distributions.

In line with \cite{aubert:generalpmle}, we introduce the stochastic risk function $\Kbf_{N,n}$, defined as follows. For $1 \leq N \leq n$ and any sequence of conditional densities $p = (p_t)_{1 \leq t \leq n}$, we have:
\begin{equation*}
\Kbf_{N,n}(p) = \frac{1}{n-N+1} \sum_{t=N}^n \Ebb\left[ \log \frac{p_t^\star(A_t)}{p_t(A_t)} \, \Big| \Fcal_{t-1} \right].
\end{equation*}
When $N=1$, we simply write $\Kbf_{n}(p)$. This expression represents the empirical mean of the conditional KL divergence, as the term
\begin{equation*}
\Ebb\left[ \log \frac{p_t^\star(A_t)}{p_t(A_t)} \, \Big| \Fcal_{t-1} \right]
\end{equation*}
is a predictable quantity corresponding to the KL divergence between the distributions with densities $p^\star_t$ and $p_t$ w.r.t. $\mu$, conditionally to $\Fcal_{t-1}$.

When $\Fcal_t = \sigma(A_1^t)$ for $t \geq 1$ and $\Fcal_0$ is the trivial $\sigma$-algebra, the quantity $(n-N+1)\Ebb[\Kbf_{N,n}(p)]$ precisely equals the KL divergence between the distributions defined by $p^\star$ and $p$.

Similarly, we define the empirical mean of the Hellinger distance, denoted $\Hbf_{N,n}^2$, as:
\begin{equation*}
\Hbf_{N,n}^2(p) = \frac{1}{n-N+1} \sum_{t=N}^n \Ebb\left[ \DHL^2(p_t^\star(A_t), p_t(A_t))  \big| \Fcal_{t-1} \right]
\end{equation*}
with $\DHL^2(\cdot,\cdot)$ the squared Hellinger distance.


\section{HOLD-OUT PROCEDURE}\label{section_holdout}

In this section, we assume to have access to a family of sequences of conditional densities $(p^m)_{m \in \mathcal{M}}$, where $\mathcal{M}$ is a finite set with at least two elements ($|\mathcal{M}| \geq 2$), as is commonly encountered in hold-out procedures \citep[Chapter 8]{massart2007concentration}. Each model of learning $m$ is characterized by a sequence of conditional densities $p^m = (p^m_t)_{t \in [n]}$, where $p^m_t$ serves as a candidate for approximating the true conditional density $p_t^\star$. Specifically, under model $m \in \mathcal{M}$, $p^m_t(.)$ represents some conditional density of $A_t$ given the past information $\Fcal_{t-1}$. The hold-out estimator $\hat{m}$ is defined as follows. Let $n > N \geq 1$ and select
\begin{equation*}
\hat{m} \in \underset{m \in \mathcal{M}}{\arg\max} \sum_{t=N}^{n} \log p_{t}^m(A_t 
).
\end{equation*}

\begin{theorem}
\label{holdouttheorem}
Assume that for all $m \in \Mcal$,  for all $x \in \Omega$, 
\begin{itemize}
    \item for all $t \in [N]$, $p_t^m(x)$ is $\Fcal_{N-1}$-measurable
    \item for all $t \in \{N,\ldots,n\}$, $p_t^m(x)$ is $\Fcal_{t-1}$-measurable.
\end{itemize}
For all $\diamondsuit > 1$, there exists $c > 0$ such that
\begin{multline*}
 \Ebb[\Hbf^2_{N,n}(p^{\hat{m}})|\Fcal_{N-1}]
     \leq \diamondsuit \inf_{m \in \Mcal} \Ebb[\Kbf_{N,n}(p^m)|\Fcal_{N-1}]
    \\+ c \frac{\log(n-N+1) + \log|\Mcal|}{n-N+1}.
\end{multline*}
\end{theorem}
This result holds for arbitrary $p^m$ as long as they are predictable w.r.t. $(\Fcal_t)_t$. In particular, it allows, as usual for hold-out, to take  $p^m=p^m_{\tilde{\theta}^m}$, where $\tilde{\theta}^m\in \underset{\theta^m \in \Theta^m} \argmax  \sum_{t=1}^{N-1} \log p_{\theta^m,t}^m(A_t)$, for some family of models $(\Theta^m)_{m \in \Mcal}$.
This result is new, since there are no hold-out results for non-independent and non-stationary data up to our knowledge. However, it is the analog of Theorem 8.9 in \citep{massart2007concentration} for this learning framework, adding only a multiplicative factor $\log n$  in the error bound. 
It justifies the use of hold-out procedures to model learning data in cognitive experiments such as \citep{mezzadri2022hold,james2023strategy}, using classical cognitive models as Alcove \citep{Kruschke1992}, Component-Cue \citep{gluck1988evaluating} or Activity-based Credit Assignment (see \citep{james2023strategy} and the references therein).

\section{
AIC-TYPE CRITERION}\label{section_penalizedmle}

The hold-out estimator does not require prior knowledge of the underlying family of densities, but in practice \citep{wilson2019ten}, it is more natural to use procedures that allow for both parameter estimation of a model and model selection, though they rely on structural assumptions about the models \citep{massart2007concentration,spok2012,aubert:generalpmle}. In this section, we provide a new application of \citep[Theorem 1]{aubert:generalpmle} to the contextual setting defined in the second example of Section~\ref{sec:some_examples}.
To simplify the framework, we consider a finite set of actions $[K]$. The filtration is therefore $\Fcal_{t} = \sigma(A_1^{t-1},X_1^t)$, where $X_t$ is the context at time $t$ that belongs to some context space $\Xcal$. 
To stress out the dependency with the context at time $t$, we write the true unknown probability of picking an action $A_t$ given $X_t$ at time $t$ as $p_t^\star(A_t | X_t)$.

We model learning through \textit{partition-based contextual bandit} algorithms \citep[Chapter 18]{lattimore_szepesvári_2020}, which allow for a straightforward comparison to the hold-out procedure in both theoretical (Section~\ref{section_prosandcons}) and empirical settings (Section~\ref{section_numericalillustrations}).
This approach is linked to the classical problem of selecting the best partition when constructing a histogram \citep{castellan,massart2007concentration}. This modeling also offers insight into how learners use contextual information to make decisions. While we focus on a simple scenario for the sake of simplicity, Appendix~\ref{section_metalearning} extends this result to more complex bandit models for metalearning.

\subsection{Partition-based Contextual Bandits: An Example of Parametric Models}\label{section_partitionbasedcontextualbandits}

\textit{Partition-based contextual bandits} \citep[Chapter 18]{lattimore_szepesvári_2020} assume that the individual partitions the context space $\Xcal$ into disjoint cells $C$. This often occurs when the individual is familiar with the contexts and has developed a personal understanding of them. This allows the individual to learn a new task only by updating one elementary and non-contextual bandit algorithm (like \texttt{Gradient Bandit}), denoted $\texttt{CellBandit}(C)$, in each context cell $C$.

Being able to select the partition used by the individual among several candidates is important to understand its behavior.
For instance, in a categorization task where contexts are objects (see Section \ref{section_numericalillustrations}), this approach reveals the perceived similarity between objects by the learner thanks to the estimated partition.

Formally, let $g_t = (g_{1,t}, \ldots, g_{K,t}) \in [0,1]^K$ represent the vector of losses (or rewards) at time $t$, which models the feedback from the environment. We do not impose specific assumptions on how losses are generated, except that $g_t$ must be $\Fcal_{t-1}$-measurable. Contexts $X_t$ may also be generated independently of past actions or may depend on them.

Each model $m \in \Mcal$ corresponds to a partition $\Pcal_m$ of $\Xcal$ into $D_m$ cells. The model is parameterized by a vector $\theta^m = (\theta_C)_{C \in \Pcal_m}$, where each $\texttt{CellBandit}(C)$ uses a procedure characterized by a parameter $\theta_C$, such as the learning rate in the \texttt{Gradient Bandit}. The resulting candidate for $p^\star$ is $p^m_{\theta^m} = (p^m_{\theta^m,t})_{t\in [n]}$.

Under model $m$, each $\texttt{CellBandit}(C)$ is updated each time $X_t \in C$ and its decisions depend only on the contexts and actions within the set $F_t(C)=\{s\in [t]: X_s\in C\}$, with cardinality $T_t^C=|F_t(C)|$. We denote the action distribution at time $t$ for $\texttt{CellBandit}(C)$ with parameter $\theta_C$ as $\pi_{C,T_t^C}^{\theta_C}$ (see Algorithm \ref{protocolcontextualbandits}). Thus, for all $t \in [n]$ and $a \in [K]$:

\begin{align}\label{eq_probmodele}
 \nonumber p_{\theta^m,t}^m(a | X_t) 
 &= \Pbb_{\theta^m}^m (A_t = a |\Fcal_{t-1}) \\
 &= \sum_{C\in \Pcal_m} \pi_{C,T_t^C}^{\theta_C}(a) \one_{X_t\in C}.
\end{align}

\begin{algorithm}[H]
\caption{Partition-based contextual bandit for model $m$ \citep{lattimore_szepesvári_2020}}\label{protocolcontextualbandits}
\begin{algorithmic}
\State \textbf{Inputs:} partition $\Pcal_m$ 
of the context space $\Xcal$,\\ \hspace{1.1cm} parameters $\theta^m=(\theta_{C})_{C \in \Pcal_m}\in \Theta^m=\underset{{C\in \Pcal_m}}{\otimes} \Theta_C$, with $\Theta_C$ compact parametric set.
\State \textbf{Initialization}: For all $C \in \Pcal_m$, for all $a \in [K]$, $\pi_{C,1}^{\theta_C}(a) = 1/K$.
\For {$t=1,2,\ldots$}
\State Learner observes context $X_t \in \Xcal$ and finds $C \in \Pcal_m$ such that $X_t \in C$.
\State Learner plays $\texttt{CellBandit}(C)$ with parameter $\theta_C$ and samples action $A_t \sim \pi_{C,T_t^C}^{\theta_C}$.
\State Learner observes loss $g_{A_t,t}$ and updates the probability distribution $\pi_{C,T_t^C}^{\theta_C}$ in $\texttt{CellBandit}(C)$.
\EndFor
\end{algorithmic}
\end{algorithm}

First, we need to assume that the probabilities do not vanish.
\begin{Assumption}\label{existenceupsilon} There exists $\varepsilon > 0$ and an integer $T_\varepsilon \geq 2$, such that, almost surely,
\begin{equation}\label{eq_minoeps_pstar}
\forall t\leq T_\varepsilon, \  \forall x\in \Xcal, \   \forall a\in [K], \quad p^\star_t(a | x) \geq \varepsilon
\end{equation}
and that for all $m \in \Mcal$ and all $C\in \Pcal_m$, the $\texttt{CellBandit}(C)$ satisfies, for all parameter $\theta_C\in \Theta_C$
\begin{equation}\label{eq_minoeps}
\forall t\leq T_\varepsilon, \  \forall a\in [K], \quad \pi_{C,T_t^C}^{\theta_C}(a) \geq \varepsilon.
\end{equation}
\end{Assumption}

Assumption~\ref{existenceupsilon} is relevant because once the true probability $p_t^\star(a | x)$ of picking an action becomes too small and the learner stops making mistakes, further improvement in parameter estimation is no longer possible. This is emphasized in \citep{aubert23}, where the authors show a counterexample demonstrating that when the probability of error (picking the wrong action) is too low, the quality of estimation cannot be enhanced.

Let $\pen : \Mcal \to \Rbb_{+}$ be a penalty function. For each $m \in \Mcal$, let $\hat{\theta}^m \in \underset{\theta^m \in \Theta^{m}}{\argmax} \; \ell_{T_\varepsilon}(p^m_{\theta^m})$ be a MLE of model $m$, with $\ell$ defined as in~\eqref{eq_loglikelihood}, and select a model $\hat{m}$ that minimizes the penalized log-likelihood stopped at $T_\varepsilon$:
\begin{equation}\label{eq_penalizedloglikelihood}
\widehat{m} \in \underset{m \in \Mcal}{\argmin} \left(-\frac{\ell_{T_\varepsilon}(p^m_{\hat{\theta}^m})}{T_\varepsilon} + \pen(m) \right).
\end{equation}
To prove oracle inequalities, we need a smoothness assumption on the parametrization of each $\texttt{CellBandit}(C)$ which can be extended to $p^m$ in Proposition~\ref{prop_atail}. Assumption~\ref{existencegronwallbis} is standard for model selection and parameter estimation \citep{massart2007concentration,spok2012}.
\begin{Assumption}\label{existencegronwallbis} With the notation of Assumption \ref{existenceupsilon}, there exists $L_\varepsilon>0$ such that, almost surely, for all $m \in \Mcal$, all $C\in \Pcal_m$, for all $\delta_C, \theta_C \in \Theta_C$, for all $t\leq T_\varepsilon$,
\begin{equation}\label{eq_lip_pi}
\sup_{a\in [K]} \left|\log\left(\frac{\pi_{C,T_t^C}^{\delta_C}(a)}{\pi_{C,T_t^C}^{\theta_C}(a)}\right)\right| \leq L_\varepsilon \|\delta_C - \theta_C\|_{2}.
\end{equation}
\end{Assumption}

\begin{proposition}\label{prop_atail}
    Assume that $p^m$ is a partition-based contextual bandit as in \eqref{eq_probmodele} and Algorithm \ref{protocolcontextualbandits} and that there exists $T_\varepsilon$ such that for all $C\in \Pcal_m$, $\texttt{CellBandit}(C)$ satisfies \eqref{eq_minoeps} and \eqref{eq_lip_pi}. 
    Then, almost surely,  for all $\theta^m,\delta^m \in \Theta^{m}$, for all $t \leq T_\varepsilon$, for all $x \in \Xcal$,  for all $a \in [K]$, $p_{\theta^m,t}^m(a,x) \geq \varepsilon$ and
    \begin{equation*}
     \sup_{a\in [K]} \left|\log\left(\frac{p_{\delta^m,t}^m(a | x)}{p_{\theta^m,t}^m(a | x)}\right)\right| \leq L_\varepsilon \sup_{C \in \Pcal_m}\|\delta_C- \theta_C\|_{2}.
    \end{equation*}
\end{proposition}
Finally, we make the following assumption.
\begin{Assumption}\label{assumption_numberofparametersbounded}
The number of parameters of all \texttt{CellBandit} procedures are uniformly bounded, so that $d = \sup_{m \in \Mcal} \sup_{C \in \Pcal_m} \dim(\Theta_C)$ is finite.
\end{Assumption}

With these assumptions, one obtains the following result.

\begin{theorem}\label{generalmodelselectiontheorem}
Let $\Mcal$ be a countable set, and for each $m \in \Mcal$, consider a partition-based contextual bandit model $\{p^m_{\theta^m}, \theta^m\in \Theta^m\}$ (see Algorithm \ref{protocolcontextualbandits} and~\eqref{eq_probmodele}).
Let $R$ and $r$ be such that all coordinates $\theta_{i,C}$'s of $\theta_C\in \Theta_C$, for $C\in \Pcal_m$ and $m\in \Mcal$, satisfy $r\leq \theta_{i,C}\leq R$ and let 
$A_\varepsilon = L_\varepsilon \sqrt{d}(R-r) + 2 \log(\varepsilon^{-1})$. Let $\Sigma_\varepsilon = \log(A_\varepsilon)\sum_{m \in \Mcal} e^{-D_m} < +\infty .$
Under Assumptions~\ref{existenceupsilon}, \ref{existencegronwallbis} and~\ref{assumption_numberofparametersbounded}, for any $\diamondsuit > 1$, there exist $c, c' > 0$ such that the following holds:
if for all $m \in \Mcal$,
\begin{equation*}
    \pen(m) \geq 
    c A_\varepsilon^2 \log(\varepsilon^{-1}) \log(T_\varepsilon A_\varepsilon)^2 \frac{D_m}{T_\varepsilon},
\end{equation*}
then,   
    \begin{multline*}
    \Ebb[\Kbf_{T_\varepsilon}(p_{\hat{\theta}^{\hat{m}}}^{\hat{m}})]
    \leq
    \\
    \Ebb\left[ \diamondsuit \inf_{m \in \Mcal} \left( \inf_{\theta^m \in \Theta^{m}} \Kbf_{T_\varepsilon}(p_{\theta^{m}}^{m})
        + 2 \pen(m)
    \right)\right]\\
        + 
            c' A_\varepsilon  \Sigma_\varepsilon \log(\varepsilon^{-1})
            \frac{\log(T_\varepsilon)}{T_\varepsilon}.
    \end{multline*}
\end{theorem}
The set of models $\Mcal$ can be any set as long as it satisfies $\sum_{m \in \Mcal} e^{-D_m} < + \infty$. This condition is standard in the classic model selection literature \citep{massart2007concentration}. For our application to partitions of $\Xcal$, we may take any subset of the family of all possible partitions, but the number of partitions with a given dimension should not grow too fast with the dimension: for instance, it is impractical to take all possible partitions of $\Xcal$ since we incur a cost $\Sigma_\varepsilon/T_\varepsilon$ in the error bound.

Theorem~\ref{generalmodelselectiontheorem} is an application of~\citet{aubert:generalpmle}.  However, despite the generality of their results, \cite{aubert:generalpmle} show mainly applications to the classical settings of time-homogeneous models. We go further by providing a ready-to-use version for partition-based contextual bandits.

This result closely resembles the  model selection ``à la Birgé-Massart" \citep[Section 7.4]{massart2007concentration}, with a bias-variance compromise and a penalty close to the variance in $D_m/T_\varepsilon$, up to additional logarithmic terms, $\log^2 T_\varepsilon$ in the penalty and $\log T_\varepsilon$ in the residual error.
It follows from the general result of \citet{aubert:generalpmle}, applicable to dependent non-stationary data, though verifying its assumptions can be tedious. Partition-based contextual bandits easily meet these assumptions, such as those in \eqref{eq_minoeps} and \eqref{eq_lip_pi} for the \texttt{CellBandit} (see Section~\ref{section_examplebandits} for examples). 

Studying the properties of MLE typically requires regularity assumptions on the parameterization of the distribution (differentiability of the likelihood \citep{spok2012}, for instance). In particular, \cite{aubert:generalpmle} assume a Lipschitz or Hölder condition on the parameterization. This condition is reflected in the case of a partition-based contextual bandit by Assumption~\ref{existencegronwallbis}. This assumption is arduous to check for contextual bandits because errors in the parameters compound over time. In Appendix~\ref{section_examplebandits}, we prove that the \texttt{Exp3-IX} and \texttt{Gradient} \texttt{bandit} algorithms satisfy them. In addition, we also apply the result of~\citet{aubert:generalpmle} in the metalearning setup (Appendix~\ref{section_metalearning}) for the Exp4 algorithm.

\subsection{Method Comparison and Limitations}\label{section_prosandcons}

The hold-out method, while assumption-free regarding the family of densities, relies on two losses—Hellinger distance and KL divergence—which are standard in model selection, although they are in general not equivalent (see \citep{massart2007concentration}, Theorem 7.11). 

Due to the data's strong dependencies, the hold-out procedure performs a single split at $t=N$ unlike classical cross-validation. A trade-off is needed: $N$ must be large enough to accurately perform model fitting but not too large to leave enough data for the model selection step. This approach is unsuited when the individual learns differently over time, such as when they change their behavior between the start and the end of the learning process (e.g. change the underlying partition in partition-based contextual bandits).

Unlike the hold-out, the AIC-type approach does not require sample splitting and performs well in practice (Section~\ref{section_numericalillustrations}). However, its oracle inequality holds only for data from the time interval $[T_\varepsilon]$ (see Appendix~\ref{section_examplebandits}). This restriction is actually quite reasonable in practice, since nothing can be estimated about the learning process once it reaches a stage where it makes no errors (see also the negative results of \cite{aubert23}).

Both model selection methods rely on a hyperparameter $N$ (the split position) or $c$ (the penalty constant), which lack theoretical guidelines and must be empirically tuned. 
In the classical i.i.d. framework, \citet{arlot2010survey} provide guidance on sample splitting in V-fold cross-validation. However, for our scenario—an individual learning task—there is no theoretical recommendation for choosing $N$ for the hold-out criterion. We need to adjust $N$ numerically as illustrated in Figure~\ref{fig:choiceofN}. 

For the AIC-type criterion, the penalty involves the constant $c$ which is unknown beforehand and requires numerical calibration as well (see Section \ref{section_numericalillustrations}). We could either use the hold-out procedure from Section \ref{section_holdout} or apply heuristics like the dimension jump method or slope heuristics \citep{baudry2012slope,arlot2019slope} to determine $c$. In Section \ref{section_numericalillustrations}, we use simulations for this calibration, as shown in Figure~\ref{fig:choiceofc}.

\section{NUMERICAL ILLUSTRATIONS}\label{section_numericalillustrations}

\begin{figure}[htbp]
\centering
\begin{subfigure}[b]{0.5\textwidth}
    \begin{tikzpicture}[scale=0.5,font=\small]
    \draw [->] (-3,3) -- (-2,4);
    \draw [->] (-3,3) -- (-3,2);
    \draw [->] (-3,3) -- (-2,3);

    \node at (-2,3.6) {Color};
    \node at (-2,2.8) {Size};
    \node at (-2.5,2) {Pattern};

    \draw (0,0) -- (3,0) -- (3,3) -- (0,3) -- (0,0);
    \draw (4,1) -- (4,4) -- (1,4);
    \draw [dashed] (4,1) -- (1,1);
    \draw [dashed] (1,4) -- (1,1);
    \draw [dashed] (0,0) -- (1,1);
    \draw (3,0) -- (4,1);
    \draw (3,3) -- (4,4);
    \draw (0,3) -- (1,4);

    \draw (6,0) -- (9,0) -- (9,3) -- (6,3) -- (6,0);
    \draw (7,4) -- (10,4);
    \draw [dashed] (7,1) -- (7,4);
    \draw (9,0) -- (10,1);
    \draw (9,3) -- (10,4);
    \draw (10,4) -- (10,1);
    \draw [dashed] (10,1) -- (7,1);
    \draw [dashed] (6,0) -- (7,1);
    \draw (6,3) -- (7,4);

    \draw [->] (1.5,-1) -- (8,-1);
    \node at (4.5,-1.5) {Shape};

    \draw[preaction={fill, blue!65!white}, pattern=north east lines, pattern color=black] (6,0) circle [radius=0.3];
    \draw[preaction={fill, blue!65!white}, pattern=north east lines, pattern color=black] (9,0) circle [radius=0.5];
    \draw[preaction={fill, red!65!white}, pattern=north east lines, pattern color=black] (7,1) circle [radius=0.3];
    \draw[preaction={fill, red!65!white}] (7,4) circle [radius=0.3];
    \draw[preaction={fill, red!65!white}] (10,4) circle [radius=0.5];
    \draw[preaction={fill, red!65!white}, pattern = north east lines] (0.7,0.7) rectangle (1.2,1.2);
    \draw[preaction={fill, red!65!white}, pattern = north east lines] (3.5,0.5) rectangle (4.4,1.4);
    \draw[preaction={fill, blue!65!white}] (2.5,2.5) rectangle (3.4,3.4);
    \draw[preaction={fill, blue!65!white}] (-0.3,2.7) rectangle (0.2,3.2);
    \end{tikzpicture}
    \caption{Representation in 4D space}
    \label{fig:4Dspace}
\end{subfigure}
\hfill
\begin{subfigure}[b]{0.4\textwidth}
    \begin{tikzpicture}[scale = 0.5,font=\tiny]
    \draw (0,0) -- (3,0) -- (3,3) -- (0,3) -- (0,0);
    \draw (4,1) -- (4,4) -- (1,4);
    \draw [dashed] (4,1) -- (1,1);
    \draw [dashed] (1,4) -- (1,1);
    \draw [dashed] (0,0) -- (1,1);
    \draw (3,0) -- (4,1);
    \draw (3,3) -- (4,4);
    \draw (0,3) -- (1,4);

    \draw (6,0) -- (9,0) -- (9,3) -- (6,3) -- (6,0);
    \draw (7,4) -- (10,4);
    \draw [dashed] (7,1) -- (7,4);
    \draw (9,0) -- (10,1);
    \draw (9,3) -- (10,4);
    \draw (10,4) -- (10,1);
    \draw [dashed] (10,1) -- (7,1);
    \draw [dashed] (6,0) -- (7,1);
    \draw (6,3) -- (7,4);

    \draw[fill=cyan] (6,0) circle (0.3cm) node[text=blue] {$A$};
    \draw[fill=cyan] (9,0) circle (0.3cm) node[text=blue] {$A$};
    \draw[fill=cyan] (7,1) circle (0.3cm) node[text=blue] {$A$};
    \draw[fill=green] (7,4) circle (0.3cm) node[text=blue] {$B$};
    \draw[fill=cyan] (10,4) circle (0.3cm) node[text=blue] {$A$};
    \draw[fill=green] (1,1) circle (0.3cm) node[text=blue] {$B$};
    \draw[fill=cyan] (4,1) circle (0.3cm) node[text=blue] {$A$};
    \draw[fill=green] (0,3) circle (0.3cm) node[text=blue] {$B$};
    \draw[fill=green] (3,3) circle (0.3cm) node[text=blue] {$B$};
    \end{tikzpicture}
    \caption{Category attribution}
    \label{fig:categoryattribution}
\end{subfigure}
\caption{Experiment presentation: classic $5$-$4$ category structure, widely used in cognition \citep{medin1978context}. In~\ref{fig:4Dspace}, the 9 objects to classify represented in a 4D space with respect to their attributes: Color, Size, Filling Pattern, and Shape. In~\ref{fig:categoryattribution}, by position in the 4D space, the category attribution (A or B).}
\label{fig:experiment}
\end{figure}
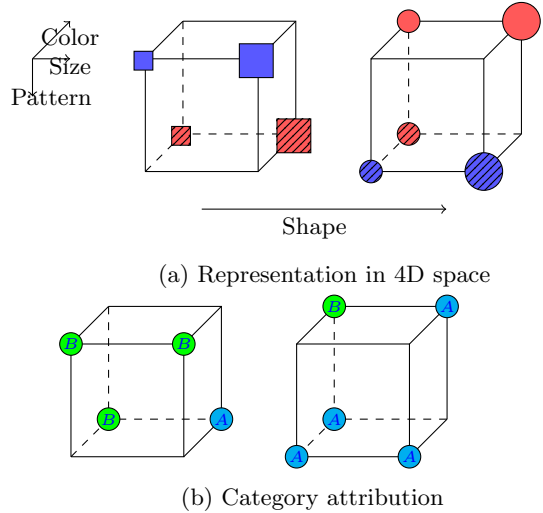

We consider an experiment on the following categorization task: learners have to classify nine objects in two categories $A$ and $B$ in a sequential way with feedback after each individual choice. Figure~\ref{fig:experiment} presents the objects and the classification rule the learners have to learn. It is a quite difficult task that has been experimented for instance in \citep{mezzadri2022investigating}, where the learners needed about 300 trials to learn the classification rule.

\begin{figure}[htbp]
    \centering
    \begin{subfigure}[b]{0.3\textwidth}
        \centering
        \includegraphics[width=1.\textwidth]{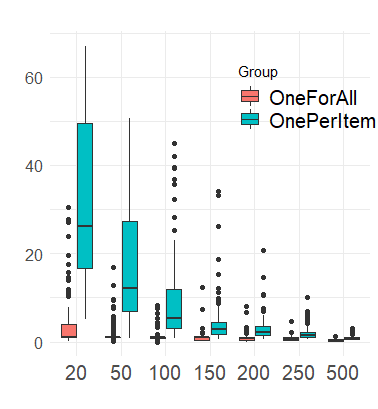}
        \caption{}
        \label{fig:boxplot_estimationerror}
    \end{subfigure}
    \hfill
    \begin{subfigure}[b]{0.3\textwidth}
        \centering
        \includegraphics[width=1\textwidth]{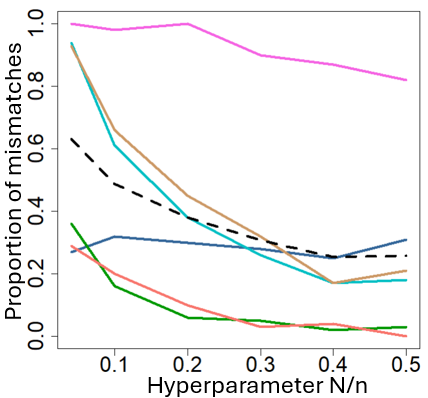}
        \caption{}
        \label{fig:choiceofN}
    \end{subfigure}
    \hfill
    \begin{subfigure}[b]{0.3\textwidth}
        \centering
        \includegraphics[width=1\textwidth]{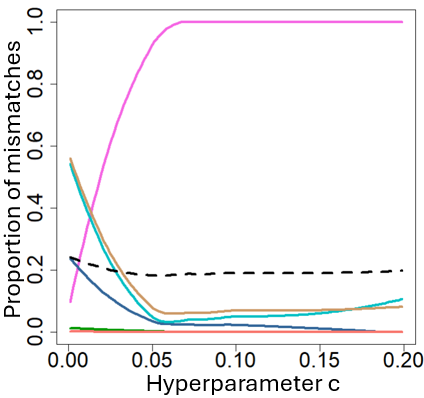}
        \caption{}
        \label{fig:choiceofc}
    \end{subfigure}
    \caption{Errors of the procedures as a function of the tuning parameters. In~\ref{fig:boxplot_estimationerror}, average of the $|\hat{\theta}_C-\theta_C|/\theta_C$ over all cells $C$ in model \texttt{OneForAll} and \texttt{OnePerItem} for the data generated respectively by the same models, where $\hat{\theta}_C$ is the MLE with likelihood truncated at $N$ (in abscissa). In~\ref{fig:choiceofN} and~\ref{fig:choiceofc}, percentage of mismatch between $\hat{m}$ and the simulated model over 100 simulations. The colors for each model are the ones given in Figure \ref{fig:boxplot_holdout_pmle} whereas the average error on the models in the dash line. In~\ref{fig:choiceofN}, for the hold-out estimator as a function of $N/n$. In~\ref{fig:choiceofc}, for the penalized MLE with $\pen(m)=c\log(n)^2D_m/n$, as a function of $c$. }
    \label{fig:threefigures}
\end{figure}

\begin{figure*}[htbp]
  \centering
  \includegraphics[width=1\textwidth]{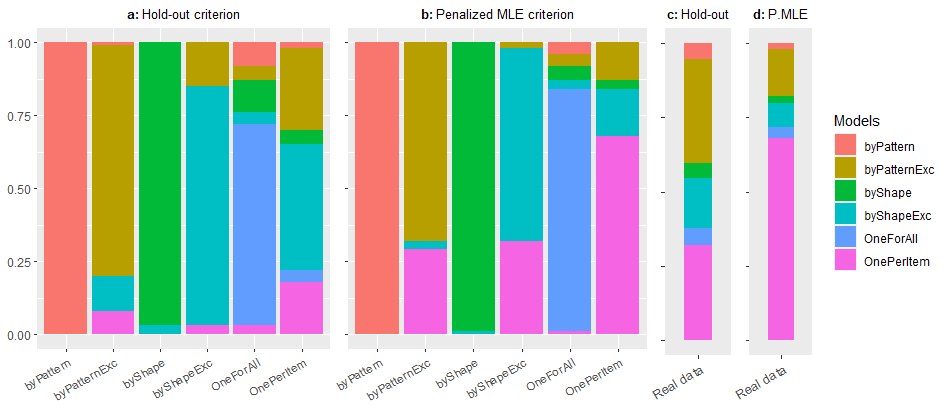} 
  \caption{Distribution of the model choices. In {\bf a},  hold-out with $N=250$ over 100 simulations. In {\bf b}, penalized MLE with $c=0.012$ over 100 simulations.  In {\bf c},  hold-out on the data recorded in \citep{mezzadri2022investigating}--176 participants. In {\bf d},  penalized MLE on the same experimental data. }
  \label{fig:boxplot_holdout_pmle}
\end{figure*}

For modeling, we fix the reward: 1 if the learner finds the good category and 0 in the other case. We focus on six models (detailed in Table~\ref{table:models} and Figure~\ref{fig:descriptionmodels} of the Appendix) with clear cognitive interpretations. Each model represents a partition of the object space, where a \texttt{CellBandit} procedure is applied within each part. For example, the \texttt{OnePerItem} model is the most complex, with the finest partition, where each element of the object space forms its own subset.

\paragraph{On Synthetic Data.}
All the simulations were performed with $n=500$ and for \texttt{Gradient Bandit} as $\texttt{CellBandit}$, for all the 6 models. The way synthetic data are generated can be found in Appendix~\ref{section_detailsnumerical}.

Figures~\ref{fig:choiceofN} and~\ref{fig:choiceofc} show that the hold-out is almost systematically outperformed by the penalized MLE. Both struggle to identify the most
complex model \texttt{OnePerItem}, preferring simpler alternatives.

\paragraph{On the choice of hyperparameters.}

The hyperparameters (c for the penalized MLE and the proportion N/n in the hold-out) can be chosen either by minimizing the average of the misclassification errors in each model or by minimizing the maximum of these errors (~\ref{fig:choiceofN} and~\ref{fig:choiceofc}). Given the nature of the plot, we chose the latter option. This led to the choice c=0.012 and N/n=0.5. We provide, in Appendix~\ref{sec:additionalplots}, additional plots with two different choices of $c$: one that minimizes the average of the errors, and the other that corresponds to the AIC, that is $c \log (n)^2 = 1$.

The models in competition do not have the same complexity. In particular, \texttt{OnePerItem} has a much larger number of parameters than the other models. The penalty function disincentivizes models with large numbers of parameters, which explains why \texttt{OnePerItem} tends to be less selected in the simulations.

However, \texttt{OnePerItem} is the model which, in real life, means that participants have learned by heart without trying to find a rule (such as treating similar shapes in the same way). It has an important meaning, especially because \cite{mezzadri2022investigating} have proven in various cognitive experiments, including this one, that some participants do not reason by rules, but by heart.

The proportion of mismatches for each model are reported in Figure~\ref{fig:boxplot_holdout_pmle}{\bf a} for the hold-out and ~\ref{fig:boxplot_holdout_pmle}{\bf b} for the penalized MLE. Both methods manage to recover the true model with less than 35\% of mistakes, except for the model \texttt{OnePerItem}, for which only the penalized MLE is able to achieve a successful match more than 60\% of the time. The models that are confused the most are the ones that are able to correctly learn the categorization, that is \texttt{ByPatternExc}, \texttt{ByShapeExc} and \texttt{OnePerItem}.

Each method has a small bias: the penalized estimator prefers the \texttt{OnePerItem} model and has a tendency to select it even when it is wrong (Figure~\ref{fig:boxplot_holdout_pmle}{\bf b}), whereas the hold-out favors \texttt{ByShapeExc} (~\ref{fig:boxplot_holdout_pmle}{\bf a}).

\paragraph{On Real Data.} 

The data have been collected in \citep{mezzadri2022investigating}\footnote{We refer the reader to \citep{mezzadri2022investigating} for precise description of the task as well as the ethics agreement.} and we focus only on the learning data. We use only the 176 participants that needed at least $n=100$ trials. In Figure~\ref{fig:boxplot_holdout_pmle}\textbf{d}, we see that most of participants are attributed one of the 3 models able to learn. The most frequent is \texttt{OnePerItem} (about 70\% for the penalized MLE) and this percentage is larger than the one obtained on simulation, probably meaning that a significant proportion of the participants do not see the division along the dimensions \texttt{Shape} or \texttt{Pattern}. It would be interesting for further study to see if this is linked to the presentation order of the objects, as it has been proved for Alcove and Component Cue in~\citep{mezzadri2022investigating}.

\section{CONCLUSION}

We proposed two selection methods of models of learning that satisfy oracle inequalities: the hold-out method  selects the best estimator in any finite family of estimators  up to logarithmic terms; the AIC-like method selects  the best trade-off between the model bias and the number of parameters of this model, under some parametrization assumptions. In all cases, the statistical challenge is that an individual's learning data are dependent and non stationary, so that classical statistical model selection results do not apply.
Future work involves a refinement of  the model class to capture more complex learning phenomena (see as premices, the Appendix~\ref{section_metalearning} on metalearning).

\subsection*{Acknowledgements}

This work was supported by the French government, through CNRS, the UCA$^{Jedi}$ and 3IA C\^ote d'Azur Investissements d'Avenir managed by the National Research Agency (ANR-15
IDEX-01 and ANR-19-P3IA-0002),  directly by the ANR project ChaMaNe (ANR-19-CE40-0024-02) and by the interdisciplinary Institute for Modeling in Neuroscience and Cognition (NeuroMod).

\bibliography{bibliography}

\section*{Checklist}



 \begin{enumerate}

 \item For all models and algorithms presented, check if you include:
 \begin{enumerate}
   \item A clear description of the mathematical setting, assumptions, algorithm, and/or model. [Yes] The notations are described in Section~\ref{sec:frame}. The set of models considered are also described in this section and specified for the bandit application in Section~\ref{section_penalizedmle}. Assumptions for the penalized log-likelihood criterion are also provided in this section. The algorithms that satisfy these assumptions are provided in Section~\ref{section_examplebandits} of the Appendix.
   \item An analysis of the properties and complexity (time, space, sample size) of any algorithm. [Yes] Details about running time and complexity are available in Section~\ref{section_codeandatadescription} of the Appendix.
   \item Anonymized source code, with specification of all dependencies, including external libraries. [Yes] The code for the experiments is available in the zip file attached to the submission. The simulation of data according to each model is precisely explained in Section~\ref{section_codeandatadescription} of the Appendix, and the code to produce is given. The codes to compute both estimators (hold-out and penalized MLE) are given and commented, also in the supplementary material. The real data are taken from a published paper \citep{mezzadri2022investigating} and we asked the authors of \citep{mezzadri2022investigating} to provide us with these data. We don't think it is possible to make these data public because the ethics agreement that has been signed by the authors of \citep{mezzadri2022investigating}, prior to data collection, might not include this possibility. However, this application to real data is mainly to prove that this can be done in practice. Since there is not a truth to be compared to in these data and since even cross validation is hard to perform, this real dataset cannot be used as a benchmark to compare methods anyway.
 \end{enumerate}

 \item For any theoretical claim, check if you include:
 \begin{enumerate}
   \item Statements of the full set of assumptions of all theoretical results. [Yes] 
   \item Complete proofs of all theoretical results. [Yes] 
   \item Clear explanations of any assumptions. [Yes]     We stated precisely all our theoretical results with assumptions that are clearly referenced. We give intuition about how each assumption is used. The proofs are given in the supplementary material.
 \end{enumerate}

 \item For all figures and tables that present empirical results, check if you include:
 \begin{enumerate}
   \item The code, data, and instructions needed to reproduce the main experimental results (either in the supplemental material or as a URL). [Yes]  All codes to generate synthetic data and to perform penalized MLE and hold-out are provided, so that all the numerical part on the calibration of both methods can be faithfully reproduced.
    Only the real dataset, as explained before, cannot be given for reproduction (Figure~\ref{fig:boxplot_holdout_pmle}).

   \item All the training details (e.g., data splits, hyperparameters, how they were chosen). [Yes] The whole purpose of our numerical study in Section~\ref{section_numericalillustrations} is to explain the choice of hyperparameters (such as the splitting in the hold-out of the calibration of the constant $c$ in the penalty).
   
    \item A clear definition of the specific measure or statistics and error bars (e.g., with respect to the random seed after running experiments multiple times). [Yes] We have run our simulations on 600 independent simulated learners and we show with a boxplot (Figure~\ref{fig:boxplot_estimationerror}) and mismatch proportion graphs (Figure~\ref{fig:choiceofN}, \ref{fig:choiceofc} and~\ref{fig:boxplot_holdout_pmle}) the proportion of erroneous selections. This cannot be done on real data, since each real participant to the categorization task is unique.
    \item A description of the computing infrastructure used. (e.g., type of GPUs, internal cluster, or cloud provider). [Yes] It is not central in our analysis so it is just mentionned in the supplementary material in Appendix~\ref{section_codeandatadescription}

 \end{enumerate}

 \item If you are using existing assets (e.g., code, data, models) or curating/releasing new assets, check if you include:
 \begin{enumerate}
   \item Citations of the creator If your work uses existing assets. [Yes] We clearly stated that the real data come from \citep{mezzadri2022investigating}. The code has been developed by us solely, using classical packages in \texttt{R} that are clearly mentioned in the code and supplementary material.
   \item The license information of the assets, if applicable. [Not Applicable]
   \item New assets either in the supplemental material or as a URL, if applicable. [Not Applicable] We do not provide new packages associated to our results.
   \item Information about consent from data providers/curators. [Yes] Authors from \citep{mezzadri2022investigating} accepted that we use their datasets to perform our simulations.
   \item Discussion of sensible content if applicable, e.g., personally identifiable information or offensive content. [Not Applicable] This work is theoretical. The methods that are validated theoretically here have already been in use in practice for a long time (see for instance the rules to follow for cognitive modeling in \citep{wilson2019ten}) and so the expected societal impact of the present work is negligible.
 \end{enumerate}

 \item If you used crowdsourcing or conducted research with human subjects, check if you include:
 \begin{enumerate}
   \item The full text of instructions given to participants and screenshots. [No] We did not collect data for the present article but used data collected for \citep{mezzadri2022investigating}, a work that is already published. In this article, all the details about instructions are given and we do not think it makes sense to reproduce it here for our illustration. We only kept the main description of the task so that the readers can understand what was done.
   \item Descriptions of potential participant risks, with links to Institutional Review Board (IRB) approvals if applicable. [No] The data collection done for \citep{mezzadri2022investigating} had the approval of the local ethic committee as mentioned in their article. Here we do not feel necessary to reproduce this here but rather point towards \citep{mezzadri2022investigating} for additional information about the task and its ethic agreement. 
   \item The estimated hourly wage paid to participants and the total amount spent on participant compensation. [No] In \citep{mezzadri2022investigating}, the details about wages are given.
 \end{enumerate}
\end{enumerate}

\appendix
\onecolumn
\begin{center}
    {\Large SUPPLEMENTARY MATERIAL}
\end{center}

This document contains all the additional material for the article.

\section{EXAMPLES OF \texttt{CELLBANDITS}}\label{section_examplebandits}

In this section we provide examples of $\texttt{CellBandit}$ satisfying \eqref{eq_minoeps} and \eqref{eq_lip_pi}. All the algorithms below are written for a cell $C$ and a $\texttt{CellBandit}(C)$ parameterized by $\theta_C\in \Theta_C$ compact subset of $\Rbb^d$ such that $R\geq \sup_{\theta_C\in \Theta_C} \|\theta_C\|_\infty$. 

\subsection{Example 1: \texttt{Exp3-IX} }\label{exampleExp3}

This algorithm is a generalization of \texttt{Exp3} and was introduced in \citep{neu2015explore}. 
Following \citep{aubert23}, we write \texttt{Exp3-IX} with parameters decreasing as a square root of the sample size to ensure a good MLE estimation of the parameters. Note in addition that, for \texttt{Exp3} and its variants, it is well known that sublinear convergence of the regret occurs when the learning rate $\eta$ and the exploration term $\gamma$ are decreasing as a square root of the sample size. This renormalization ensures that the learner is able to learn at a good pace and at the same time be robust to changes in the environment.

\begin{algorithm}[htbp]
\caption{\texttt{Exp3-IX}\citep{neu2015explore} as a $\texttt{CellBandit}(C)$}\label{Exp3-IX (Loss version)}
\begin{algorithmic}
\State \textbf{Inputs: } $n$ (Sample size), $\theta_C = (\eta, \gamma) \in \Theta_C$ (Parameter), $K$ (Number of actions).
\State \textbf{Initialization: } $\pi_{C,1}^{\theta_C}=  \left(\frac{1}{K}, \ldots, \frac{1}{K}\right)$.
\For {$t \in F_n(C)$, the set of times where $X_s\in C$,}
\State Draw an action $A_t \sim \pi_{C,T_t^C}^{\theta_C}$ and receive a loss $g_{A_t,t} \in [0,1]$.
\State Update for all $a \in [K]$,
\begin{equation*}
    \pi_{C,T_t^C+1}^{\theta_C}(a) = \frac{\exp\left(-\frac{\eta}{\sqrt{n}} \sum_{s \in F_t(C)} \hat{g}_{a,s}^{\theta_C}\right)}{\sum_{b\in [K]} \exp\left(-\frac{\eta}{\sqrt{n}} \sum_{s \in F_t(C)} \hat{g}_{b,s}^{\theta_C}\right)} \quad \text{ where } \hat{g}_{b,s}^{\theta_C} = \frac{g_{b,s}}{\gamma/\sqrt{n} + \pi_{C,T_s^C}^{\theta_C}(b)} \one_{A_s = b}
\end{equation*}
\EndFor
\end{algorithmic}
\end{algorithm}

In this case, 
$\Theta_C \subset \Rbb^2$. When $\gamma=0$, we recover the classical \texttt{Exp3} algorithm, studied from the MLE point of view in \citep{aubert23}.
Note that while $g_{A_t,t}$ is observed and known, the estimated loss $\hat{g}_{b,s}^{\theta}$ depends on the parameterization. The following result shows that one can choose \texttt{Exp3-IX} as a $\texttt{CellBandit}$ in the partition-based contextual bandits to perform partition selection.

\begin{proposition}\label{hypotheseexistenceupsilonExp3}
Let $\varepsilon \in (0,1/K)$ and let $\Theta_C\subset [0,R]^2$ with $R>0$. Then $\mathtt{Exp3-IX}$ can be a $\mathtt{CellBandit}(C)$ with parameterization $\theta_C\in \Theta_C$ that satisfies  \eqref{eq_minoeps} and \eqref{eq_lip_pi}, as soon as
    \begin{equation*}
    T_\varepsilon = \left \lfloor \left(\frac{1}{K}-\varepsilon\right)\frac{\sqrt{n}}{R} \right \rfloor \wedge n \qquad \text{and} \qquad  L_\varepsilon = \frac{\sqrt{R^2 / n + \varepsilon^2}}{\varepsilon^3 R}e^{1/\varepsilon^2}.
    \end{equation*}
\end{proposition}
This shows that one can apply Theorem \ref{generalmodelselectiontheorem} with \texttt{Exp3-IX} as $\texttt{CellBandit}$ as long as we stop using observations after $\sqrt{n}$ time steps. The dependence in $\varepsilon$ in not very critical, since it has been proved at least for \texttt{Exp3} in \citep{aubert23}, that in practice, we may take $\varepsilon$ quite large (non-vanishing) with almost no impact on $T_\varepsilon$. This is a good thing since the theoretical dependency of $L_\varepsilon$ in $\varepsilon$ is quite pessimistic.

{\it Limitations.} 
This algorithm considers the horizon $n$ fixed in order to renormalize the parameterization. From Proposition~\ref{hypotheseexistenceupsilonExp3}, it follows that Theorem~\ref{generalmodelselectiontheorem} holds when only the first $\sqrt{n}$ observations are used in the MLE, but this in no way means that the estimator will perform poorly when based on all data. Taking $\sqrt{n}$ observations compounds on the usual issue that if the number of cells is large, only a small amount of data may be available for each cell, making estimation difficult. 

\subsection{Example 2: \texttt{Gradient Bandit}}\label{examplegradientbandits}

\texttt{Gradient Bandit} is another possible algorithm. We still choose for similar reason a parameterization in $\eta/\sqrt{n}$, which echoes the Robbins-Monro conditions \citep{robbins1951stochastic} even if  \citep{mei2023stochastic} proved convergence in a stochastic bandit framework even for non renormalized parameters.

\begin{algorithm}[H]
\caption{\texttt{Gradient Bandit} \citep{mei2023stochastic} as a $\texttt{CellBandit}$}\label{gradientbandits}
\begin{algorithmic}
\State \textbf{Inputs: } $n$ (Sample size), $\theta_C \in [r,R]$ (Parameter), $K$ (Number of actions).
\State \textbf{Initialization: } $\pi_{C,1}^{\theta_C}=  \left(\frac{1}{K}, \ldots, \frac{1}{K}\right)$.
\For {$t \in F_n(C)$}
\State Draw an action $A_t \sim \pi_{C,T_t^C}^{\theta_C}$ and receive a reward $g_{A_t,t} \in [ 0,1]$.
\State Update for all $a \in [K]$,
\begin{equation*}
    \pi_{C,T_t^C+1}^{\theta_C}(a) = \frac{\exp\left(-\frac{\theta_C}{\sqrt{n}} \sum_{s \in F_t(C)} \hat{g}_{a,s}^{\theta_C}\right)}{\sum_{b \in [K]} \exp\left(-\frac{\theta_C}{\sqrt{n}} \sum_{s \in F_t(C)} \hat{g}_{b,s}^{\theta_C}\right)} \quad \text{ where } \hat{g}_{b,s}^{\theta_C} =\left(\one_{A_s = b}-\pi_{C,T_s^C}^{\theta_C}(b)\right) g_{A_s,s}
\end{equation*}
\EndFor
\end{algorithmic}
\end{algorithm}

\begin{proposition}\label{existencehypotheseupsilongradientbrandits}
Let $\varepsilon \in (0,1)$ and let $\Theta_C\subset [0,R]^2$ with $R>0$. Then, $\mathtt{Gradient\; Bandit}$ can be a $\mathtt{CellBandit}(C)$ with parameterization $\theta_C\in \Theta_C$ that satisfies  \eqref{eq_minoeps} and \eqref{eq_lip_pi}, as soon as 
    \begin{equation*}
        T_\varepsilon := \left \lfloor \log\left(\sqrt{\frac{1}{K\varepsilon}}\right)\frac{\sqrt{n}}{R} \right \rfloor \wedge n \qquad \text{and} \qquad L_\varepsilon = \frac{\sqrt{2}}{R \varepsilon}\frac{\log\left(\sqrt{\frac{1}{K\varepsilon}}\right)}{\sqrt{K \varepsilon}}.
    \end{equation*}
\end{proposition}
This theoretical result has the same interpretation as before: the theoretical guarantees of Theorem \ref{generalmodelselectiontheorem} with \texttt{Gradient Bandit} as $\texttt{CellBandit}$ hold when we stop using observations after $\sqrt{n}$ time steps. In practice, we can use the observations up to time $n$ (see Section~\ref{section_numericalillustrations}).

\section{CODE AND DATA DESCRIPTION}\label{section_codeandatadescription}

The code for reproducing the figures in the article is available at the following link: https://github.com/JulienAubert3/ContextualBandits.

\subsection{Contextual Bandit Based Models}

\begin{table*}[htbp]
\centering
\caption{Description of models and their learning abilities}
\vspace{1em} 
\begin{tabular}{m{2cm} m{1cm} m{7cm} m{2cm}}
\toprule
\textbf{Model} & \textbf{Number of cells} & \textbf{Description of the cells} & \textbf{Learns categorization} \\ 
\midrule
\texttt{OneForAll} & 1 & One giant cell & No \\ 
\midrule
\texttt{ByShape} & 2 & One for circles, one for squares & Partly \\ 
\midrule
\texttt{ByPattern} & 2 & One for striped items, one for plain items & Partly \\ 
\midrule
\texttt{ByShapeExc} & 4 & Cells from \texttt{ByShape} model with exceptions isolated & Yes \\ 
\midrule
\texttt{ByPatternExc} & 4 & Cells from \texttt{ByPattern} model with exceptions isolated & Yes \\ 
\midrule
\texttt{OnePerItem} & 9 & One cell for each item & Yes \\ 
\bottomrule
\end{tabular} 
\label{table:models}
\end{table*}

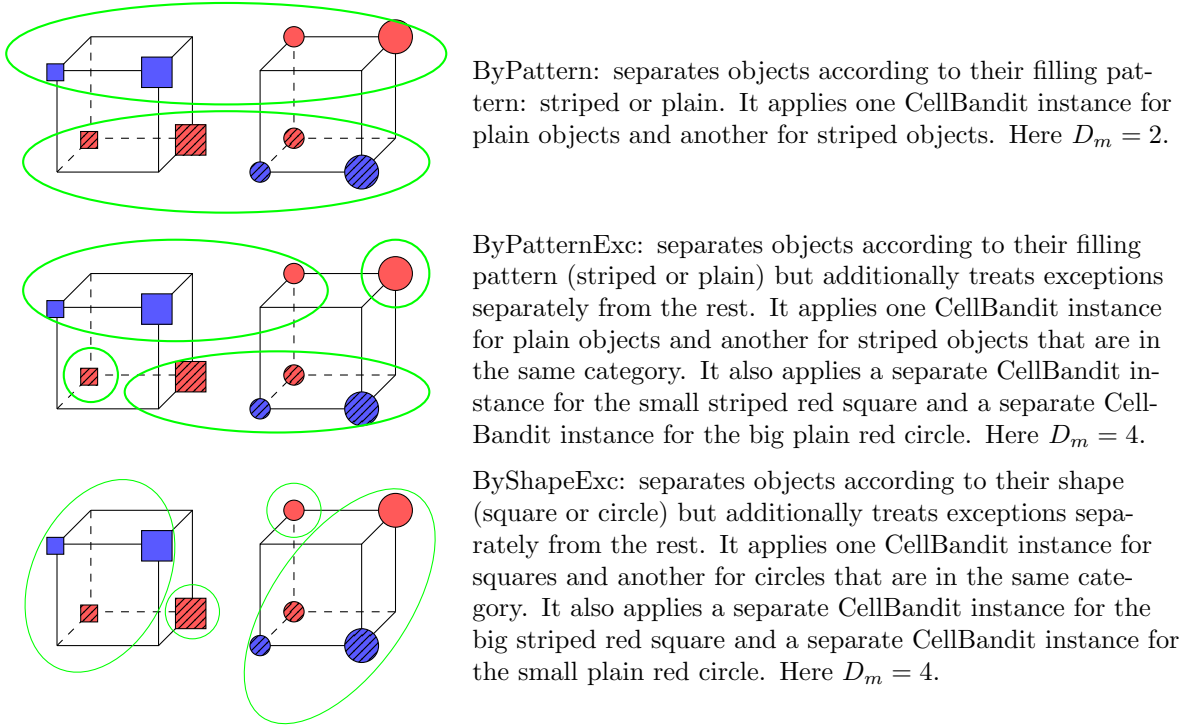
\begin{figure}[H]
\begin{tikzpicture}[scale=0.45]
        \draw (0,0) -- (3,0) -- (3,3) -- (0,3) -- (0,0);
        \draw (4,1) -- (4,4) -- (1,4);
        \draw [dashed] (4,1) -- (1,1);
        \draw [dashed] (1,4) -- (1,1);
        \draw [dashed] (0,0) -- (1,1);
        \draw (3,0) -- (4,1);
        \draw (3,3) -- (4,4);
        \draw (0,3) -- (1,4);

        \draw (6,0) -- (9,0) -- (9,3) -- (6,3) -- (6,0);
        \draw (7,4) -- (10,4);
        \draw [dashed] (7,1) -- (7,4);
        \draw (9,0) -- (10,1);
        \draw (9,3) -- (10,4);
        \draw (10,4) -- (10,1);
        \draw [dashed] (10,1) -- (7,1);
        \draw [dashed] (6,0) -- (7,1);
        \draw (6,3) -- (7,4);

        \draw[preaction={fill, blue!65!white}, pattern=north east lines, pattern color=black] (6,0) circle [radius=0.3];
        \draw[preaction={fill, blue!65!white}, pattern=north east lines, pattern color=black] (9,0) circle [radius=0.5];
        \draw[preaction={fill, red!65!white}, pattern=north east lines, pattern color=black] (7,1) circle [radius=0.3];
        \draw[preaction={fill, red!65!white}] (7,4) circle [radius=0.3];
        \draw[preaction={fill, red!65!white}] (10,4) circle [radius=0.5];
        \draw[preaction={fill, red!65!white}, pattern = north east lines] (0.7,0.7) rectangle (1.2,1.2);
        \draw[preaction={fill, red!65!white}, pattern = north east lines] (3.5,0.5) rectangle (4.4,1.4);
        \draw[preaction={fill, blue!65!white}] (2.5,2.5) rectangle (3.4,3.4);
        \draw[preaction={fill, blue!65!white}] (-0.3,2.7) rectangle (0.2,3.2);

        \node[text width=9.5cm, anchor=west] at (12,2) {ByPattern: separates objects according to their filling pattern: striped or plain. It applies one CellBandit instance for plain objects and another for striped objects. Here $D_m = 2$.};

        \draw[green, thick] (5,0.3) ellipse [x radius=6, y radius=1.5];

        \draw[green, thick] (5,3.5) ellipse [x radius=6.5, y radius=1.5];

    \begin{scope}[shift={(0,-7)}]
        \draw (0,0) -- (3,0) -- (3,3) -- (0,3) -- (0,0);
        \draw (4,1) -- (4,4) -- (1,4);
        \draw [dashed] (4,1) -- (1,1);
        \draw [dashed] (1,4) -- (1,1);
        \draw [dashed] (0,0) -- (1,1);
        \draw (3,0) -- (4,1);
        \draw (3,3) -- (4,4);
        \draw (0,3) -- (1,4);

        \draw (6,0) -- (9,0) -- (9,3) -- (6,3) -- (6,0);
        \draw (7,4) -- (10,4);
        \draw [dashed] (7,1) -- (7,4);
        \draw (9,0) -- (10,1);
        \draw (9,3) -- (10,4);
        \draw (10,4) -- (10,1);
        \draw [dashed] (10,1) -- (7,1);
        \draw [dashed] (6,0) -- (7,1);
        \draw (6,3) -- (7,4);

        \draw[preaction={fill, blue!65!white}, pattern=north east lines, pattern color=black] (6,0) circle [radius=0.3];
        \draw[preaction={fill, blue!65!white}, pattern=north east lines, pattern color=black] (9,0) circle [radius=0.5];
        \draw[preaction={fill, red!65!white}, pattern=north east lines, pattern color=black] (7,1) circle [radius=0.3];
        \draw[preaction={fill, red!65!white}] (7,4) circle [radius=0.3];
        \draw[preaction={fill, red!65!white}] (10,4) circle [radius=0.5];
        \draw[preaction={fill, red!65!white}, pattern = north east lines] (0.7,0.7) rectangle (1.2,1.2);
        \draw[preaction={fill, red!65!white}, pattern = north east lines] (3.5,0.5) rectangle (4.4,1.4);
        \draw[preaction={fill, blue!65!white}] (2.5,2.5) rectangle (3.4,3.4);
        \draw[preaction={fill, blue!65!white}] (-0.3,2.7) rectangle (0.2,3.2);

         \draw[green, thick] (6.5,0.5) ellipse [x radius=4.5, y radius=1.2];

         \draw[green, thick] (10,4) circle [radius=1];

         \draw[green, thick] (1,1) circle [radius=0.8];

        \draw[green, thick] (3.5,3.5) ellipse [x radius=4.5, y radius=1.5];

        \node[text width=9.5cm, anchor=west] at (12,2) {ByPatternExc: separates objects according to their filling pattern (striped or plain) but additionally treats exceptions separately from the rest. It applies one CellBandit instance for plain objects and another for striped objects that are in the same category. It also applies a separate CellBandit instance for the small striped red square and a separate CellBandit instance for the big plain red circle. Here $D_m = 4$.};
    \end{scope}

    \begin{scope}[shift={(0,-14)}]
        \draw (0,0) -- (3,0) -- (3,3) -- (0,3) -- (0,0);
        \draw (4,1) -- (4,4) -- (1,4);
        \draw [dashed] (4,1) -- (1,1);
        \draw [dashed] (1,4) -- (1,1);
        \draw [dashed] (0,0) -- (1,1);
        \draw (3,0) -- (4,1);
        \draw (3,3) -- (4,4);
        \draw (0,3) -- (1,4);

        \draw (6,0) -- (9,0) -- (9,3) -- (6,3) -- (6,0);
        \draw (7,4) -- (10,4);
        \draw [dashed] (7,1) -- (7,4);
        \draw (9,0) -- (10,1);
        \draw (9,3) -- (10,4);
        \draw (10,4) -- (10,1);
        \draw [dashed] (10,1) -- (7,1);
        \draw [dashed] (6,0) -- (7,1);
        \draw (6,3) -- (7,4);

        \draw[preaction={fill, blue!65!white}, pattern=north east lines, pattern color=black] (6,0) circle [radius=0.3];
        \draw[preaction={fill, blue!65!white}, pattern=north east lines, pattern color=black] (9,0) circle [radius=0.5];
        \draw[preaction={fill, red!65!white}, pattern=north east lines, pattern color=black] (7,1) circle [radius=0.3];
        \draw[preaction={fill, red!65!white}] (7,4) circle [radius=0.3];
        \draw[preaction={fill, red!65!white}] (10,4) circle [radius=0.5];
        \draw[preaction={fill, red!65!white}, pattern = north east lines] (0.7,0.7) rectangle (1.2,1.2);
        \draw[preaction={fill, red!65!white}, pattern = north east lines] (3.5,0.5) rectangle (4.4,1.4);
        \draw[preaction={fill, blue!65!white}] (2.5,2.5) rectangle (3.4,3.4);
        \draw[preaction={fill, blue!65!white}] (-0.3,2.7) rectangle (0.2,3.2);

        \draw[green](7,4) circle [radius = 0.8];
        \draw[green](4,1) circle [radius = 0.8];
        \draw[green, rotate around={-25:(2,1.5)}] (1.1,1.7) ellipse [x radius=2, y radius=3];
        \draw[green, rotate around={-35:(7,1.5)}] (8.3,2) ellipse [x radius=2, y radius=4];

        \node[text width=9.5cm, anchor=west] at (12,2) {ByShapeExc: separates objects according to their shape (square or circle) but additionally treats exceptions separately from the rest. It applies one CellBandit instance for squares and another for circles that are in the same category. It also applies a separate CellBandit instance for the big striped red square and a separate CellBandit instance for the small plain red circle. Here $D_m = 4$.};
    \end{scope}
\end{tikzpicture}
\caption{Design of models ByPattern, ByPatternExc and ByShapeExc. Exceptions are determined by the categorization rule in Figure 1b of the article.}\label{fig:descriptionmodels}
\end{figure}

\begin{figure}[htbp]
    \centering
    \begin{subfigure}[b]{0.48\textwidth}
        \centering
        \includegraphics[width=1\textwidth]{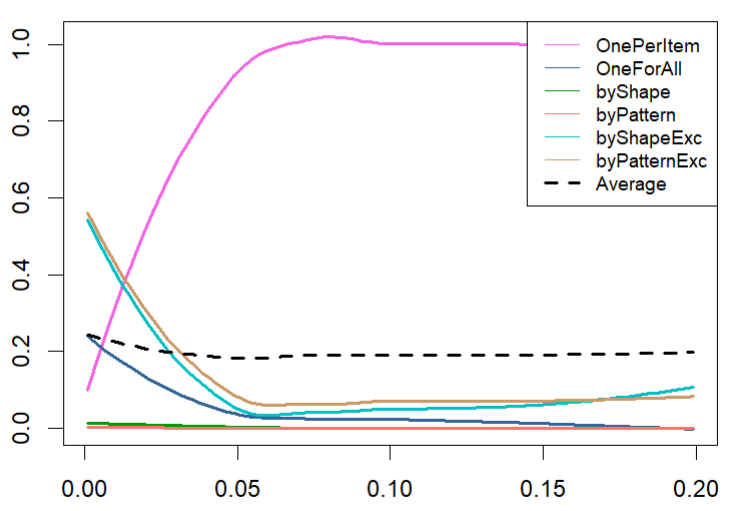}
    \end{subfigure}
    \hfill
    \begin{subfigure}[b]{0.48\textwidth}
        \centering
        \includegraphics[width=1\textwidth]{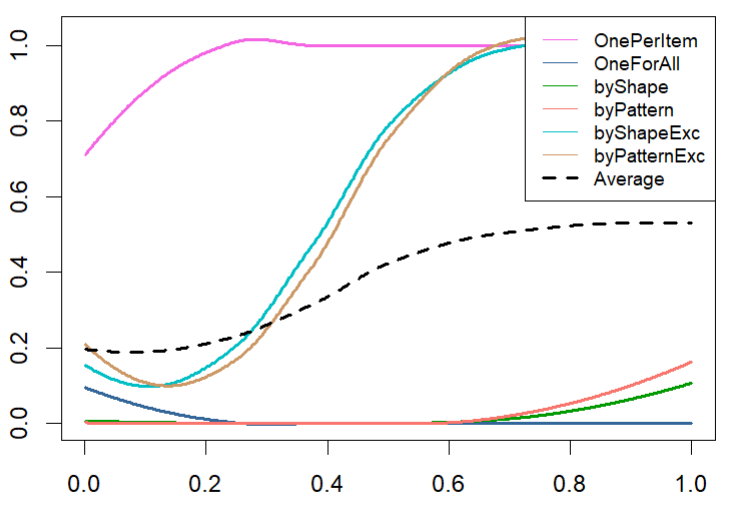}
    \end{subfigure}
    \caption{Errors of the penalized log-likelihood criterion as a function of the tuning parameter $c$, where $c$ is such that $\pen(m)=c\log(n)^2D_m/n$ is a function of $c$. Both figures show the percentage of mismatches between $\hat{m}$ and the simulated model over 100 simulations. The same simulations were used for both figures. The evolution of errors as a function of $c$ is logical in relation to the value of $D_m$.}
    \label{fig:threefiguresbis}
\end{figure}

\subsection{Additional Plots}\label{sec:additionalplots}

In this section, we provide additional plots (like Figure~\ref{fig:boxplot_holdout_pmle}) for two values of the tuning parameter $c$.
\begin{itemize}
    \item The value of $c$ minimizing the average of all errors for the penalized log likelihood criterion. As noted in Figure~\ref{fig:choiceofc}, this corresponds to the choice $c = 0.053$. This choice strongly penalizes the model OnePerItem, which is almost never selected, even when it should be.
    \begin{figure}
        \centering
        \includegraphics[width=0.7\linewidth]{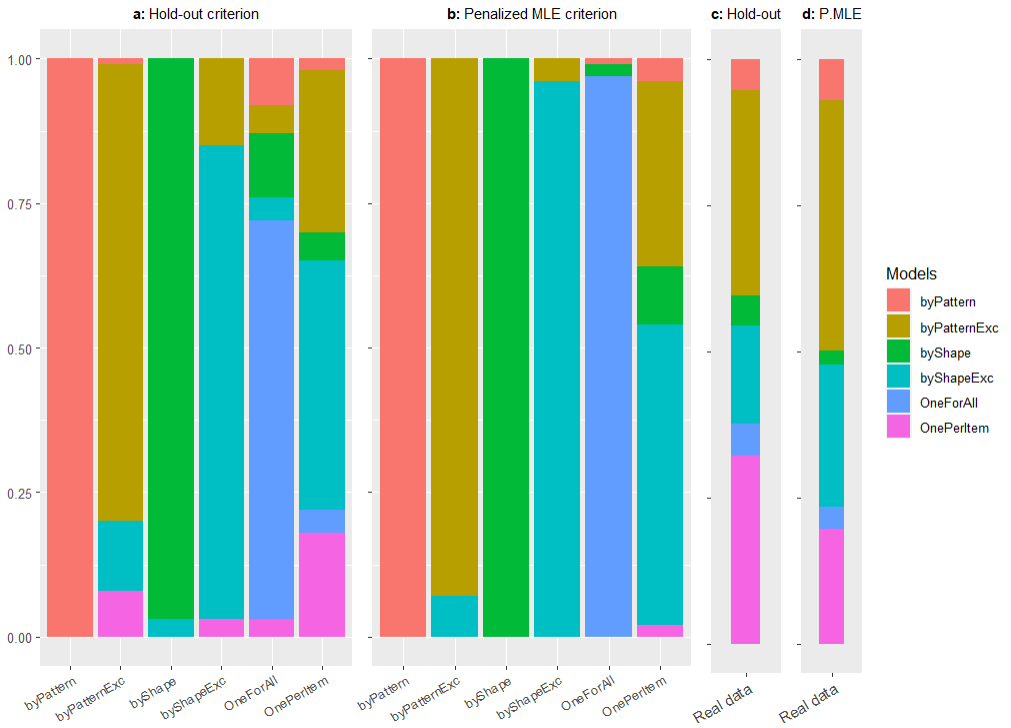}
        \caption{Penalized log-likelihood procedure and Hold-out on synthetic and real data, for the choice $N=250$ and $c=0.53$.}
        \label{fig:boxplot_c053}
    \end{figure}
    \item The value of $c$ corresponding to the AIC criterion, that is: $c \log(n)^2 = 1$.
    \begin{figure}
        \centering
        \includegraphics[width=0.7\linewidth]{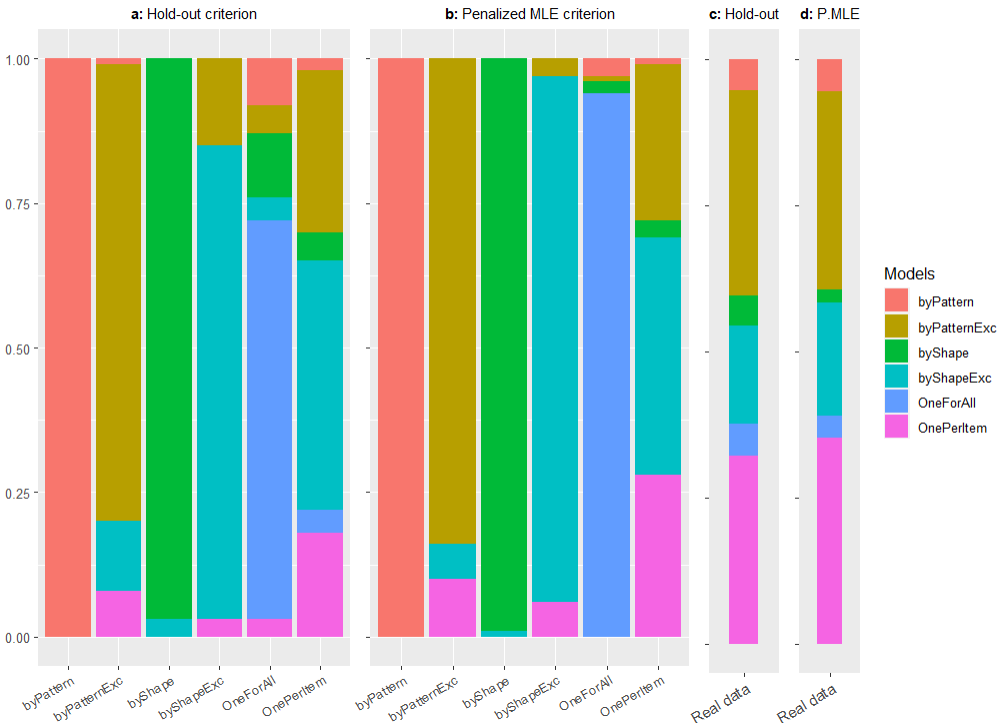}
        \caption{Penalized log-likelihood procedure and Hold-out on synthetic and real data, for the choice $N=250$ and $c \log (n)^2 = 1$.}
        \label{fig:boxplot_cAIC}
    \end{figure}
\end{itemize}

\subsection{Difference Between The Empirical and Theoretical Threshold}

We have not been able to run the simulation with \texttt{Exp3-IX}. Indeed, as also shown practically in \citep{aubert23} for the simple \texttt{Exp3} case, the probabilities $\pi_{C,T_t^C}^{\theta_C}$ can go to zero extremely fast. When the individual learns over an horizon $n = 500$, only $\sqrt{n}=22$ observations would be usable and the estimations even of just the MLE is unreliable. So all the simulations were performed with $n=500$ and for \texttt{Gradient Bandit} as $\texttt{CellBandit}$, for all the 6 models described in Table~\ref{table:models}. The way synthetic data are generated can be found in Appendix~\ref{section_detailsnumerical}.

Figure~\ref{fig:boxplot_estimationerror} shows that despite the conservative theoretical bound given in Proposition \ref{existencehypotheseupsilongradientbrandits} with $T_\varepsilon$ of order $\sqrt{n}$, \texttt{Gradient Bandit} provides good results when the MLE is applied to all $n$ data points. 
The truncation at $\sqrt{n}\simeq 20$ required in the theoretical results does not seem necessary in practice, and actually looks suboptimal for \texttt{Gradient Bandit}.

\subsection{Details on Numerical Illustrations}\label{section_detailsnumerical}

In this section, we give details on the numerical illustrations of Section~\ref{section_numericalillustrations}. The images were obtained using the \texttt{ggplot2} package of \texttt{R}. Two types of analyses were conducted, 
on synthetic data and 
on real data.

\paragraph{On Synthetic Data.} The simulations of the  synthetic data helped us calibrate the tuning parameters choices for the hold-out and the penalized log-likelihood procedure. In Section~\ref{section_holdout}, the parameter $N$ must be calibrated for choosing the correct training data sample size. In Section~\ref{section_penalizedmle}, as said in the \textit{Limitations}, since the constant $c$ in the penalty term is not known a priori, it must be calibrated as well. To do this, we follow the guidelines of \cite{wilson2019ten}. The procedure is as follows.

\begin{itemize}
    \item[1)]{Sample size:} $n = 500$. It is of the same order of magnitude as real data.
    
    \item[2)]{Objects generation:} periodic sequence of the nine objects repeated through the $n$ trials. We generate a sequence of objects following the same structure as in \citep{mezzadri2022investigating}. Due to the periodic pattern, each object is therefore seen roughly the same number of times for all time $t$. 
    \item[3)]{Actions generation:}  for each model in Table~\ref{table:models}, we generated $100$ sequences of actions called synthetic agents with respect to the procedure given in Algorithm~\ref{protocolcontextualbandits} with \texttt{Gradient Bandit} as \texttt{CellBandit}.  The parameters $\theta_C$ we used were the same for each model and the same for each cell, equal to $0.03\times \sqrt{n}$, except for the \texttt{OnePerItem} model where we changed slightly the values of the parameter in each cell to make the model identifiable. For $m=$\texttt{OnePerItem}, we took $\theta^m = ((0.03/10+k\times 0.007)\times \sqrt{n})_{k \in \{0,\ldots,8\}}$ following the same order of presentation of the sequence of objects defined earlier.
    \item[4)]{Parameters estimation: } we then fitted each of the six models on all the synthetic agents generated data, and we estimated the associated parameters using (MLE) and the package \texttt{DEoptim} in \texttt{R} with range $(0, 1)$ for the parameters $\theta_C/\sqrt{n}$  
    and with the default parameters
    and a maxiter value equal to $20$. We then computed the log-likelihood associated to the estimated parameters. We did this for the likelihood stopped at time $N \in \{25,50,100,150,200,250,n\}$.
    \begin{itemize}
        \item[-] With such data, we were able to plot Figure~\ref{fig:boxplot_estimationerror} and Figure~\ref{fig:choiceofN} with the hold-out criterion defined in Section~\ref{section_holdout}. In Figure~\ref{fig:boxplot_estimationerror}, we computed the average error made in each cell by the model fitting of the same model that generated the data. For the Figure~\ref{fig:choiceofN}, we simply counted the number of times each model verified the hold-out criterion for all the synthetic agents and for each model that generated the data.
        \item[-] With the log-likelihood stopped at time $n$ for the estimated parameters, we were able to plot Figure~\ref{fig:choiceofc} according to the penalized log-likelihood criterion defined in~\eqref{eq_penalizedloglikelihood}. In the same way we counted the number of times each model satisfied the penalized log-likelihood criterion for all the synthetic agents and for each model that generated the data.
    \end{itemize}
    \item[5)]{Choice of the parameters $N$ and $c$ for the real data}: Given the results of Figure~\ref{fig:choiceofN} and Figure~\ref{fig:choiceofc}, we chose to use $N$ to be equal to half of the data length and $c=0.012$ to account for a reasonable error for   model \texttt{OnePerItem}, even if in average $c=0.04$ gives better results. With this data, we were able plot the two first chart of Figure~\ref{fig:boxplot_holdout_pmle}.
\end{itemize}

\paragraph{On Real Data.} For the real experimental data, here is the process we followed.
\begin{itemize}
    \item[1)] Sample size: dependent on each individual, the average data sample size is $300$.
    \item[2)] Objects and Actions: we collected for each individual their objects sequence and associated choices.
    \item[3)] For each individual, we fitted the $6$ models and estimated the parameters associated to each model. To perform hold-out and penalized log-likelihood model selection, we used the parameters $N$ and $c$ chosen thanks to the synthetic data. With this data, we were able to plot Figure~\ref{fig:boxplot_holdout_pmle}.
\end{itemize}

\subsection{About the Code and the Data}

In this section, we give explanations about the code and data (e.g. computation time, link between code and data). All the data, code and images used are provided in the zip file associated to submission, called \texttt{ContextualBanditsCode}. We run all the simulations in $\texttt{R}$ and used the following packages: \texttt{DEoptim}, \texttt{crayon}, \texttt{magrittr}, \texttt{dplyr}, \texttt{tidyr}, \texttt{ggplot2}, \texttt{gridExtra}.

For the sample size we chose, all the simulations can run on a PC in a reasonable time of execution (detailed hereafter). Overall, computing the different data and running the code took approximately 6 hours excluding the time needed for the real data. The biggest file is 
373 kilobytes. The PC we used was a Gigabyte - AORUS 15G XC, with processor: Intel(R) Core(TM) i7-10870H CPU \@ 2.20GHz, 2208 MHz, 8 cores, 16 logical processors.

\paragraph{On Real Data.} As mentioned earlier, we could not provide the experimental data used in \citep{mezzadri2022investigating}, since they have already been published in another paper and we do not want to break the ethic agreement. We can only provide the results and estimated data resulting from the experimental data. Note however that the procedures to obtain the following RData files are the same as for the synthetic data which we detail later. The three RData files on the real data are \texttt{realdatamle}, \texttt{realdata\_holdout\_trainingset}, \texttt{realdata\_holdout\_testingset}.
\begin{itemize}
    \item \texttt{realdatamle} is a list of estimators and associated log-likelihood for each model and each individual.
    \item \texttt{realdata\_holdout\_trainingset} is a list of estimators and associated log-likelihood on the first half of the sample for each model and each individual.
    \item \texttt{realdata\_holdout\_testingset} is a list of log-likelihood on the testing part of the sample for each individual and each model with parameters estimated in \texttt{realdata\_holdout\_trainingset}.
\end{itemize}

\paragraph{On Synthetic Data.} All the synthetic data obtained in the other files can be computed by running the code \texttt{ContextualbanditsCodebis}. The code is commented and starts with a list of functions which are necessary to run the different procedures. In the code, we explain how the different procedures lead to the following list of files. We have commented with \# the parts of the code that would modify the files so that running the code now would give the same images as the ones used in the article. If one wants to generate new data, one should uncomment these lines of code.  However, we advise the reader that some of the procedures take a certain time, and would recommend not to do so. We detail hereafter the content of the different csv and RData files and the time it took to run them. 
\begin{itemize}
    \item To begin with, we generate a csv file called \texttt{databis\_500.csv} of $500$ trials and associated list of objects in the file \texttt{synthetic\_data}.
    \item In the same \texttt{synthetic\_data} file we create the different model files and within each of them generate $100$ csv files of actions, rewards, and objects according to the procedure described in~\ref{section_detailsnumerical}. This procedure takes around 5 minutes. Then, we begin to compute the MLE for each of the synthetic data csv file.
    \item \texttt{Datalikelihood100agents6modeletabis500horizon} is a nested list of estimators, associated log-likelihood stopped at time $n$ for each model fitted to the data of all the synthetic agents. Computing these data took approximately 2 hours . 
    \item \texttt{holdoutbis100agents6models\_horizon\_20} is the same nested list of estimators but computed on a log-likelihood stopped at time $N=20$. Computing these data took approximately 10 minutes .
    \item \texttt{holdoutbis100agents6models\_horizon\_50} is the same nested list of estimators but computed on a log-likelihood stopped at time $N=50$. Computing these data took approximately 20 minutes .
    \item \texttt{holdoutbis100agents6models\_horizon\_100} is the same nested list of estimators but computed on a log-likelihood stopped at time $N=100$. Computing these data took approximately 30 minutes .
    \item \texttt{holdoutbis100agents6models\_horizon\_150} is the same nested list of estimators but computed on a log-likelihood stopped at time $N=150$. Computing these data took approximately 40 minutes .
    \item \texttt{holdoutbis100agents6models\_horizon\_200} is the same nested list of estimators but computed on a log-likelihood stopped at time $N=200$. Computing these data took approximately 50 minutes .
    \item \texttt{holdoutbis100agents6models\_horizon\_250} is the same nested list of estimators but computed on a log-likelihood stopped at time $N=250$. Computing these data took approximately 1 hour .
    \item \texttt{alldataholdoutbis} is a nested list of errors on estimation for the training data and log-likelihood function on the testing data for all synthetic agents, all models and all training data sample size $N \in \{20,50,100,150,200,250\}$. Computing these data took approximately 10 minutes .
\end{itemize}

\section{METALEARNING}\label{section_metalearning}

By looking at the experiment in Section~\ref{section_numericalillustrations}, it is hard to believe that learners start directly with a model like \texttt{ByPatternExc}. It is more likely that they start with a model like \texttt{ByPattern} and realize that there are too many exceptions, so that they progressively end up with \texttt{ByPatternExp}. One way to model this progressive switch from one strategy to the other is to use bandits with expert advice. In this framework, there is a finite set $E$ of randomized policies called experts, $(\xi_{j,t}(.))_{t\in[n]}$, probabilities over the set of actions $[k]$, that are modeling the different strategies the learner might have. No assumptions are made here on the way experts compute their randomized predictions: they might be the result of contextual bandits like \texttt{ByPattern} or more generally any kind of computations that depend on the learner's past choices. 
\texttt{Exp4} (see Algorithm~\ref{alg_expertadvice}) is an adaptation of \texttt{Exp3} to this case (see \citep{lattimore_szepesvári_2020} for regret convergence and variants such as \texttt{Exp4.P} \citep{beygelzimer2011contextual}). 

\begin{algorithm}[H]
\caption{Exp4 \citep{BubeckBianchi}}\label{alg_expertadvice}
\begin{algorithmic}
\State \textbf{Inputs: } $n$ (Sample size), $\theta \in [r,R]$ (Parameter), $K$ (Number of actions), $E$ (Set of experts).
\State \textbf{Initialization: } $q_{E,1}^{\theta}$ uniform distribution over the experts $E$.
\For {$t=1,2,\ldots$}
\State Receive experts advice $a\mapsto\xi_{j,t}(a)$  probability distribution over $[K]$ for all $j$.
\State Draw an action $A_t \sim \pi_{E,t}^{\theta}(.) = \sum_{j\in E} q_{E,t}^\theta(j) \xi_{j,t}(.)$ and receive a reward $g_{A_t,t} \in [0,1]$.
\State Update for all $j\in E$,
\begin{equation*}
    q_{E,t+1}^\theta(j) = \frac{\exp\left(-\frac{\theta}{\sqrt{n}} \sum_{s=1}^t \hat{y}_{j,s}^\theta\right)}{ \sum_{i \in E} \exp\left(-\frac{\theta}{\sqrt{n}} \sum_{s=1}^t \hat{y}_{i,s}^\theta\right)}\quad\mbox{with}\quad \hat{y}_{i,s}^\theta =  \sum_{a \in[K]} \xi_{i,t}(a) \frac{g_{a,s}}{\pi_{E,s}^\theta(a)}\one_{A_s=a}
\end{equation*}
\EndFor
\end{algorithmic}
\end{algorithm}

In this setting, a model $m$ is defined by a finite set $E_m$ representing the different experts/strategies the learner is learning from. Since there is only one parameter by model (namely $\theta\in [r,R]$), the  penalty plays no role, nor the calibration of $c$. So there is no need for hold-out and one can prove that the model with the smallest log-likelihood on the first $T_\varepsilon\sim \sqrt{n}$ time steps satisfies an oracle inequality if $\Mcal$ is finite, as well as $|F| := \max_{m \in \Mcal} |E_m|$. This shows that one can select the set $E_m$ of strategies which is the closest to reality among the sets of strategies that are put in competition. 

{\it Limitations.} The only limitation with this approach is that we need at first to know the eventual parameters of each strategy. Again we could split the data in a hold-out fashion to make the injection of estimated parameters possible. However, it would be then nearly impossible to correctly estimate the parameters of strategies that are not used at the beginning of the learning. We refer to \citep{binz2023meta} for other methods in meta-learning for cognition.



Since we work in a more general setting and not simply with contexts, we assume that we observe a process $(A_t)_{1 \leq t \leq n}$ adapted to a general filtration $(\Fcal_t)_{1 \leq t \leq n}$ where for all $t \in [n]$, $A_t \in [K]$. In particular, for all $t \in [n]$, $\Fcal_t$ is generated by the past actions $(A_1,\ldots,A_{t})$ and any other additional variable which might be observed or not -- such as a context at time $t+1$ for instance. We write, for all $a \in [K]$, and all $t \in [n]$
\begin{equation*}
    p_t^\star(a) = \Pbb(A_t = a | \Fcal_{t-1})
\end{equation*}
the true conditional distribution we wish to estimate. 

Additionally, we consider the family of models $\{\pi_{E_m}^{\theta^m} = (\pi_{E_m,t}^{\theta^m})_{t \in [n]}, m \in \Mcal\}$ where $\Mcal$ is a finite set, $\theta^m \in [r,R]$, and for all $m \in \Mcal$, $(\pi_{E_m,t}^{\theta^m})_{t \in [n]}$ is the sequence of mixtures of probability distributions over actions defined recursively in Algorithm~\ref{alg_expertadvice} for the finite set $E_m$. Each model $m$ is thus defined by a set of experts $(\xi_{j,t}(.))_{j \in E_m,t\in[n]}$ where for all $m \in E_m$, $t \in [n]$, $\xi_{j,t}$ can be any probability distribution over arms $[K]$ as long as it is measurable with respect to $\Fcal_{t-1}$.

Let $|F| := \max_{m \in \Mcal} |E_m|$. The goal is to select the set $E_m$ of policies -- that we see as learning strategies --  with which the individual learns to learn. This approach is again based on partial log-likelihood $\ell_{n}(\pi_{E_m}^{\theta^m})$ of the observations $(A_1,\ldots,A_n)$ defined by
\begin{equation}\label{eq_loglikelihood_expert}
\ell_{n}(\pi_{E_m}^{\theta^m}) = \sum_{t=1}^{n} \log\left(\pi_{E_m,t}^{\theta^m}(A_t)\right).
\end{equation}

To achieve a model selection result, we need the following assumption on the policies given by the experts.

\begin{Assumption}\label{hyp_expert_model}
There exists $\rho > 0$, such that almost surely, for all $m \in \Mcal$, for all $t \in [n]$ and all $i \in [K]$, $\sum_{j \in E_m} \xi_{j,t}(i) \geq \rho.$
\end{Assumption}

Then, with Assumption~\ref{hyp_expert_model}, we can deduce a result similar to Propositions~\ref{hypotheseexistenceupsilonExp3} and~\ref{existencehypotheseupsilongradientbrandits} because of the very structure of Algorithm~\ref{alg_expertadvice} which mimics \texttt{Exp3}.

\begin{proposition}\label{prop_expert_minorant}
    Assume Assumption~\ref{hyp_expert_model} holds. Let $\rho$ be the associated constant. Let $\varepsilon \in (0,\rho/|F|)$, and let
    \begin{equation*}
        T_\varepsilon = \left \lfloor \left( \frac{1}{|F|}-\frac{\varepsilon}{\rho}\right) \frac{\sqrt{n}}{R} \right \rfloor \wedge n \qquad \text{and} \qquad L_\varepsilon=\frac{1}{R \varepsilon^2} \exp\left(\frac{1}{\varepsilon^2}\right).
    \end{equation*}
    Then, for all $t \in [T_\varepsilon]$, for all $m \in \Mcal$, $\theta^m,\delta^m \in [r,R]$, for all $k \in [K]$,
    \begin{equation*}
        \pi_{E_m,t}^{\theta^m}(k) \geq \varepsilon \qquad \text{and} \qquad \sup_{k \in [K]} \left| \log\left(\frac{\pi_{E_m,t}^{\theta^m}(k)}{\pi_{E_m,t}^{\delta^m}(k)} \right) \right| \leq L_\varepsilon |\theta^m - \delta^m|.
    \end{equation*}
\end{proposition}

Finally, we still assume that the true distribution is bounded away from $0$ (as in~\eqref{eq_minoeps_pstar}).

\begin{Assumption}\label{hyp_expert_true}
Assume that Assumption~\ref{hyp_expert_model} holds. Let $\varepsilon$ and $T_\varepsilon$ be the constants of Proposition~\ref{prop_expert_minorant}. Assume that
\begin{equation*}
    \forall t \leq T_\varepsilon, \forall a \in [K], p_t^\star(a) \geq \varepsilon.
\end{equation*}
\end{Assumption}
Assumptions~\ref{hyp_expert_model} and~\ref{hyp_expert_true} allow us to verify Assumptions 1 and 2 of \citep{aubert:generalpmle}. As for Section~\ref{section_penalizedmle}, it is thus possible to put into competition different sets of experts.
Let $A_\varepsilon = L_\varepsilon(R-r)+2\log(\varepsilon^{-1})$. Since all the models have the same dimension, there is no penalty term to account for. So the term $\Sigma_\varepsilon$ in Theorem~\ref{generalmodelselectiontheorem} becomes $\log(A_\varepsilon) |\Mcal| e^{-1}$.  The result of \citep[Corollary 2]{aubert:generalpmle} states that for any $\diamondsuit > 1$, there exist constants $c,c'$ such that,

\begin{multline*}
   \Ebb \left[\Kbf_{T_\varepsilon}(p^\star,\pi_{E_{\hat{m}}}^{\hat{\theta}^{\hat{m}}})\right]
    \leq \Ebb\left[ \diamondsuit \inf_{m \in \Mcal} \left(  \inf_{\theta \in \Theta^{D_m}} \Kbf_{T_\varepsilon}( p^\star, \pi_{E_m}^{\theta^m})
    \right)\right] + \frac{c}{\kappa} A_\varepsilon^2 \log(\varepsilon^{-1}) \log(T_\varepsilon A_\varepsilon)^2 \frac{1} {T_\varepsilon}\\
        + \frac{18 e^{-1} c'}{\kappa} 
            A_\varepsilon \log(A_\varepsilon) |\Mcal| \log(\varepsilon^{-1})
            \frac{\log(T_\varepsilon)}{T_\varepsilon}.
\end{multline*}

\section{PROOFS}\label{section_proofs}

In all of the following, we use the  conditional Kullback-Leibler divergence $\DKL$ between $p_t^\star(\cdot)$ and $p_{t}^m(\cdot)$:
\begin{equation*}
\DKL\left(p_t^\star(A_t), p_{t}^m(A_t)\right) = \Ebb \left[ \log \frac{p_t^\star(A_t)}{p_{t}^m(A_t)}\Big|\Fcal_{t-1} \right].
\end{equation*}

\subsection{Proof of Section~\ref{section_holdout}}\label{sectionproofholdout}

\begin{proof}[Proof of Theorem~\ref{holdouttheorem}]

In the sequel to avoid losing track of important dependence, we denote: for any $t \in [n]$, $m \in \Mcal$, $x_t \mapsto p^m_{t}(x_t |\Fcal_{t-1})$ the density of $A_t$ conditionally to $\Fcal_{t-1}$, and likewise for $p^\star_t(x_t |\Fcal_{t-1} )$.

Consider the following functions, defined for all $t \in [n]$, $x_1^t \in \Xcal^t$, $m \in \Mcal$ by
\begin{equation*}
      \left\{
    \begin{aligned}
      & g_t^m(A_t,\Fcal_{t-1}) = -\frac{1}{2} \log\left(\frac{p_{t}^m(A_t|\Fcal_{t-1})}{p_{t}^*(A_t|\Fcal_{t-1})}\right)\\
      & f_t^m(A_t,\Fcal_{t-1}) = - \log\left(\frac{p_{t}^*(A_t|\Fcal_{t-1})+ p_{t}^m(A_t|\Fcal_{t-1})}{2 p_{t}^*(A_t|\Fcal_{t-1})}\right).
    \end{aligned}
  \right.
\end{equation*}
For any family $h = (h_t)_{t \in [n]}$ of functions $A_t$ that may depend on the past, that is $h_t(A_t,\Fcal_{t-1})$, we write
\begin{equation*}
   \left\{
    \begin{aligned}
      & \tilde{\nu}(h) = \frac{1}{n-N+1} \sum_{t=N}^{n} \left(h_t(A_t,\Fcal_{t-1}) -\Ebb[h_t(A_t,\Fcal_{t-1})|\Fcal_{t-1}]\right),\\
      & \tilde{P}(h) = \frac{1}{n-N+1} \sum_{t=N}^{n} h_t(A_t,\Fcal_{t-1}),\\
      & \tilde{C}(h) = \frac{1}{n-N+1} \sum_{t=N}^{n} \Ebb[h_t(A_t,\Fcal_{t-1})|\Fcal_{t-1}].
    \end{aligned}
  \right.
\end{equation*}
Let $m \in \Mcal$. From the definition of $\hat{m}$,
\begin{align*}
     \tilde{P}\left(g^{\hat{m}}\right)
    \leq \tilde{P}\left(g^{m}\right).
\end{align*}
Since, $f^{\hat{m}}(A_t,\Fcal_{t-1}) \leq g^{\hat{m}}(A_t,\Fcal_{t-1})$ by concavity of $\log$, it holds that
\begin{align*}
    \tilde{\nu}\left( f^{\hat{m}}\right) +  \tilde{C}\left(f^{\hat{m}}\right) = \tilde{P}\left(f^{\hat{m}}\right)
      \leq \tilde{P}\left(g^m\right) = \tilde{\nu}\left( g^{m}\right) +  \tilde{C}\left(g^{m}\right).
\end{align*}
That is
\begin{align*}
    \tilde{\nu}\left( f^{\hat{m}}-f^{m}\right) +  \tilde{C}\left(f^{\hat{m}}\right)
      \leq \tilde{\nu}\left( g^{m}-f^{m}\right) +  \tilde{C}\left(g^{m}\right).
\end{align*}
Let $U_m = \tilde{\nu} \left(g^{m}-f^{m}\right)$, then
\begin{align}\label{preholdout}
       \tilde{C}\left(f^{\hat{m}}\right)
      \leq \tilde{C}\left(g^{m}\right)-\tilde{\nu}\left( f^{\hat{m}}-f^{m}\right)+U_m.
\end{align}
Note that $U_m$ is centered. For $m' \in \Mcal$, let $M^{m'}_{N} = 0$, and for $t \geq N+1$, let
\begin{equation*}
    M^{m'}_t = -\sum_{s=N}^{t-1}\left(f_s^{m'}(A_s,\Fcal_{s-1})-f_s^{m}(A_s,\Fcal_{s-1})-\Ebb\left[ f_s^{m'}(A_s,\Fcal_{s-1})-f_s^{m}(A_s,\Fcal_{s-1})|\Fcal_{s-1}\right]\right).
\end{equation*}
Then, $(M^{m'}_t)_{t \geq N}$ is an $(\Fcal_{t})_{t \geq N}$-martingale.
For $\ell \geq 2$, let $B_{N}^{\ell}=0$, and for $t \geq N+1$, let
\begin{equation*}
    B_t^{\ell} := \sum_{s=N}^{t-1} \Ebb \left[\left(M^{m'}_{s+1} - M^{m'}_{s}\right)^{\ell} \, \bigg| \, \Hcal_{s} \right].
\end{equation*}
For $t \in \{N,\ldots,n-1\}$, note that
\begin{equation*}
|M^{m'}_{t+1}-M^{m'}_t| \leq 2 \int
 \left| f_t^{m'}(x_t,\Fcal_{t-1})-f_t^{m}(x_t,\Fcal_{t-1}) \right| \frac{d\delta_{A_t}(x_t) + p^\star_t(x_t | \Fcal_{t-1})d\dommes(x_t)}{2},
\end{equation*}
so that, by convexity of $x \mapsto |x|^\ell$,
\begin{equation*}
   |M^{m'}_{t+1}-M^{m'}_t| \leq 2^\ell \int
 \left| f_t^{m'}(x_t,\Fcal_{t-1})-f_t^{m}(x_t,\Fcal_{t-1}) \right|^\ell \frac{d\delta_{A_t}(x_t) + p^\star_t(x_t | \Fcal_{t-1})d\dommes(x_t)}{2}.
\end{equation*}
Thus,
\begin{align}\label{partialBt}
   \nonumber B_t^{\ell} & \leq 2^\ell \sum_{s=N}^{t-1}   \int
\Ebb \left[ \left| f_s^{m'}(x_s,\Fcal_{s-1})-f_s^{m}(x_s,\Fcal_{s-1}) \right|^\ell \frac{d\delta_{X_s}(x_s) + p^\star_s(x_s | \Fcal_{s-1})d\dommes(x_s)}{2} \big| \Fcal_{s-1}\right]\\
   \nonumber & = 2^\ell \sum_{s=N}^{t-1} \int
 \left| f_s^{m'}(x_s,\Fcal_{s-1})-f_s^{m}(x_s,\Fcal_{s-1}) \right|^\ell  p^\star_s(x_s | \Fcal_{s-1}) d \dommes(x_s) \\
 &  = 2^\ell \sum_{s=N}^{t-1} \int
 \left|\log \left( \frac{p_s^*(x_s|\Fcal_{s-1}) + p_s^{m'}(x_s|\Fcal_{s-1})}{p_s^*(x_s|\Fcal_{s-1}) + p_s^{m}(x_s|\Fcal_{s-1})}\right) \right|^\ell  p^\star_s(x_s | \Fcal_{s-1})d \dommes(x_s).
\end{align}
We now need the following Lemma to continue.
\begin{lemma}\citep[Lemma 7.26]{massart2007concentration}\label{puissancedulog}
    For all $\ell \geq 2$ and all $x > 0$,
    \begin{equation*}
        \frac{|\log(x)|^\ell}{\ell!} \leq \frac{9}{64} \left(x -\frac{1}{x}\right)^2.
    \end{equation*}
\end{lemma}

\begin{proof}
    The complete Lemma and proof of the Lemma can be found in \citep{massart2007concentration}.
\end{proof}
Applying Lemma~\ref{puissancedulog} to $x = \displaystyle \sqrt{\frac{p_s^*(x_s|\Fcal_{s-1}) + p_s^{m'}(x_s|\Fcal_{s-1})}{p_s^*(x_s|\Fcal_{s-1}) + p_s^{m}(x_s|\Fcal_{s-1})}}$ leads to, for all $x_s$, 
\begin{multline*}
    \left|\log\left(\frac{p_s^*(x_s|\Fcal_{s-1}) + p_s^{m'}(x_s|\Fcal_{s-1})}{p_s^*(x_s|\Fcal_{s-1}) + p_s^{m}(x_s|\Fcal_{s-1})}\right)  \right|^{\ell} \\
    \leq \frac{9}{64} 2^{\ell} \ell! \frac{\left(p_{s}^{m}(x_s|\Fcal_{s-1})-p_{s}^{m'}(x_s|\Fcal_{s-1})\right)^2}{\left(p_{s}^{m}(x_s|\Fcal_{s-1})+ p_{s}^*(x_s|\Fcal_{s-1})\right)\left(p_{s}^*(x_s|\Fcal_{s-1})+p_{s}^{m'}(x_s|\Fcal_{s-1})\right)}.
\end{multline*}
Plugging this in Equation~\eqref{partialBt} leads to
\begin{equation*}
    |B_t^{\ell}|
    \leq \frac{9}{4} 2^{2(\ell-2)} \ell! \sum_{s=N}^{t-1} \int\frac{\left(p_{s}^{m}(x_s|\Fcal_{s-1})-p_{s}^{m'}(x_s|\Fcal_{s-1})\right)^2 p_{s}^*(x_s|\Fcal_{s-1})}{\left(p_{s}^{m}(x_s|\Fcal_{s-1})+ p_{s}^*(x_s|\Fcal_{s-1})\right)\left(p_{s}^*(x_s|\Fcal_{s-1})+p_{s}^{m'}(x_s|\Fcal_{s-1})\right)} d\dommes(x_s).
\end{equation*}
For all $x,y,z \geq 0$,
\begin{equation*}
    \left(\sqrt{x}+\sqrt{y}\right)^2 z \leq (z+y)(z+x),
\end{equation*}
therefore, with $z=p_{s}^*(x_s|\Fcal_{s-1})$, $x=p_{s}^{m}(x_s|\Fcal_{s-1})$ and $y=p_{s}^{m'}(x_s|\Fcal_{s-1})$,
\begin{align}\label{finalBt}
    |B_t^{\ell}| \leq \frac{9}{4} 4^{\ell-2} \ell!\sum_{s=N}^{t-1} \int\left(\sqrt{p_{s}^{m}(x_s|\Fcal_{s-1})}-\sqrt{p_{s}^{m'}(x_s|\Fcal_{s-1})}\right)^2 d\dommes(x_s) \leq 
    \frac{1}{2}4^{\ell-2} \ell! V^{m'}_t,
\end{align}
where
\begin{align}\label{defVt}
    \nonumber V^{m'}_t &:= \frac{9}{2}\sum_{s=N}^{t-1} \int\left(\sqrt{p_{s}^{m}(x_s|\Fcal_{s-1})}-\sqrt{p_{s}^{m'}(x_s|\Fcal_{s-1})}\right)^2 d\dommes(x_s) \\
    & = 9 \sum_{s=N}^{t-1} H \left(p_{s}^{m}(A_s|\Fcal_{s-1}),p_{s}^{m'}(A_s|\Fcal_{s-1})\right)^2 
\end{align}
where $H$ is the Hellinger distance between the two probability distributions $p_{s}^{m}(A_s|\Fcal_{s-1})$ and $p_{s}^{m'}(A_s|\Fcal_{s-1})$.
Lemma 3.3 of \citep{Houdre} gives that for all $\lambda > 0$,
\begin{equation*}
    \Ecal_t = \exp \left(\lambda M^{m'}_t- \sum_{\ell \geq 2} \frac{\lambda ^{\ell}}{\ell !} B_t^{\ell}\right)
\end{equation*}
is a supermartingale and that in particular $\Ebb(\Ecal_{n+1}) \leq 1$. By Equation~\eqref{finalBt}, for $\lambda \in (0,1/4)$,
\begin{align*}
   \sum_{\ell \geq 2} \frac{\lambda ^{\ell}}{\ell !} B_t^{\ell} \leq \frac{\lambda^2}{2}\sum_{\ell \geq 2} (4\lambda)^{\ell-2} V^{m'}_t = \frac{\lambda^2}{2(1-4\lambda)} V^{m'}_t.
\end{align*}
Let $\Psi(\lambda) = \frac{\lambda^2}{2(1-4\lambda)}$ for $\lambda \in (0,1/4)$. Then,
\begin{equation*}
    \Ebb\left[ e^{\lambda M^{m'}_{n+1} - \Psi(\lambda) V^{m'}_{n+1}}  \, \bigg| \, \Fcal_{N-1}\right] \leq 1.
\end{equation*}
By Markov's inequality, for all $x \geq 0$ and $\lambda \in (0,1/4)$,
\begin{equation}\label{markovmartingalebisbis}
    \Pbb \left(M^{m'}_{n+1} \geq V^{m'}_{n+1}\frac{\Psi(\lambda)}{\lambda} + \frac{x}{\lambda} \, \bigg| \, \Fcal_{N-1} \right) \leq e^{-x}.
\end{equation}

Therefore, for all $x, u \geq 0$ and $\lambda \in (0,1/4)$,
\begin{equation*}
    \Pbb \left(M^{m'}_{n+1} \geq \frac{\Psi(\lambda)}{\lambda} u  + \frac{x}{\lambda}\quad \text{and} \quad V^{m'}_{n+1} \leq u  \, \bigg| \, \Fcal_{N-1}\right) \leq e^{-x}.
\end{equation*}

To choose the optimal $\lambda$, we use Lemma 2 from \citep{Hansen_2015}.
\begin{lemma}\citep[Lemma 2]{Hansen_2015}\label{hansen}
    Let $a,b$ and $x$ be positive constants and let us consider on $(0,1/b)$,
    \begin{equation*}
        g(\xi) = \frac{a \xi}{1-b \xi} + \frac{x}{\xi}.
    \end{equation*}
Then $\min_{\xi \in (0,1/b)} \displaystyle g(\xi) = 2 \sqrt{ax} + bx$ and the minimum is achieved in $\xi(a,b,x) 
= \displaystyle \frac{\sqrt{x}}{\sqrt{x} b + \sqrt{a}}$.
\end{lemma}

For $a = \frac{u}{2}$ and $b = 4$, Lemma~\ref{hansen} shows that for all $x, u \geq 0$,
\begin{equation*}
    \Pbb \left(M^{m'}_{n+1} \geq \sqrt{2 u x}+4x\quad \text{and} \quad V^{m'}_{n+1} \leq u \, \bigg| \, \Fcal_{N-1} \right) \leq e^{-x}.
\end{equation*}

Let us use a peeling argument similar to \citep{Hansen_2015}:

\begin{lemma}
\label{lemma_peeling}
Let $X,V$ be real-valued random variables and $\alpha,b,v,w$ be positive numbers such that $V \in [w,v]$ a.s. and such that for all $x \geq 0$ and $u \in [w,v]$,
\begin{equation*}
    \Pbb( X \geq \sqrt{ux} + bx \quad \text{and} \quad (1+\alpha)^{-1} u \leq V \leq u ) \leq e^{-x},
\end{equation*}
then for any $x \geq 0$,
\begin{equation*}
    \Pbb( X \geq \sqrt{(1+\alpha)Vx} + bx) \leq \left( 1 + \frac{\log(v/w)}{\log(1+\alpha)} \right) e^{-x}.
\end{equation*}
\end{lemma}

\begin{proof}
Let $v_0 = w$, $v_{d+1} = (1+\alpha) v_d$, and $D$ the smallest integer such that $v_D \geq v$. For all $d \in [D]$ and $x \geq 0$,
\begin{equation*}
    \Pbb \left(X \geq \sqrt{v_d x} + bx\quad \text{and} \quad v_{d-1} \leq V \leq v_d \right) \leq e^{-x}.
\end{equation*}
In particular, since $V \geq v_{d-1} = (1+\alpha)^{-1} v_d$ on this event,
\begin{equation*}
    \Pbb \left(X \geq \sqrt{(1+\alpha)V x} + bx\quad \text{and} \quad v_{d-1} \leq V \leq v_d \right) \leq e^{-x}.
\end{equation*}
Taking the union bound,
\begin{equation*}
    \Pbb \left(X \geq \sqrt{(1+\alpha)V x} + bx \right) \leq D e^{-x},
\end{equation*}
and by definition $D \leq \displaystyle \frac{\log(v/w)}{\log(1+\alpha)}+1$.
\end{proof}

We may apply Lemma~\ref{lemma_peeling} to $X = M^{m'}_{n+1}$ and $b=4$. Since $V^{m'}_{n+1}$ does not have an obvious lower bound, we consider $V = 2(V^{m'}_{n+1} + \beta)$ for some $\beta > 0$ to be chosen later. We may therefore take $w = 2\beta$. For the upper bound $v$ on $V$, by~\eqref{defVt}, since the Hellinger distance is upper bounded by 1, we may take $v = 2(\beta + 9 (n-N+1))$. With these choices, for any $\beta, \alpha, x >0$, 
\begin{align*}
\Pbb \left(M^{m'}_{n+1} \geq \sqrt{2(1+\alpha)(V^{m'}_{n+1} +\beta)x}+4x  \, \bigg| \, \Fcal_{N-1}\right)
    \leq \left( \frac{\log\left(\frac{9(n - N+1)}{\beta} + 1\right)}{\log(1+\alpha)}+1 \right) e^{-x}.
\end{align*}
For $\alpha = \sqrt{2}$, 
\begin{align}\label{controlnuf1bis}
\Pbb \left(M^{m'}_{n+1} \geq \sqrt{5(V^{m'}_{n+1} +\beta)x}+4x \, \bigg| \, \Fcal_{N-1}\right)
    \leq \left( 2\log\left(\frac{9(n - N+1)}{\beta} + 1\right)+1 \right) e^{-x}.
\end{align}

By definition of $V^{m'}_{n+1}$ and the triangle inequality,
\begin{align}\label{maj1Vt}
\nonumber
V^{m'}_{n+1} &= 9 \sum_{s=N}^{n} H \left(p_{s}^{m}(A_s|\Fcal_{s-1}),p_{s}^{m'}(A_s|\Fcal_{s-1})\right)^2 \\
    & \leq 18 \sum_{s=N}^{n}\left( H\left(p_{s}^*(A_s|\Fcal_{s-1}),p_{s}^{m}(A_s|\Fcal_{s-1})\right)^2+H\left(p_{s}^*(A_s|\Fcal_{s-1}),p_{s}^{m'}(A_s|\Fcal_{s-1})\right)^2\right).
\end{align}

We now use \citep[Lemma 7.23]{massart2007concentration} giving a connection between the Kullback-Leibler divergence $\DKL$ and the Hellinger distance $H$.
\begin{lemma}\citep[Lemma 7.23]{massart2007concentration}\label{lowerboundKLHellinger}
    Let $P$ and $Q$ be some probability measures. Then, 
    \begin{equation*}
        \DKL \left(P, \frac{P+Q}{2} \right) \geq \left(2 \log 2 -1\right) H^2(P,Q).
    \end{equation*}
    Moreover, whenever $P \ll Q$,
    \begin{equation*}
        2 H^2(P,Q) \leq \DKL(P,Q).
    \end{equation*}
\end{lemma}

Since $\displaystyle \frac{18}{2\log(2)-1}\leq 48$,
\begin{align}
\nonumber
 V^{m'}_{n+1}
    \leq & \, 48 \sum_{s=N}^{n} \left( \DKL\left(p_{s}^*(A_s|\Fcal_{s-1}), \frac{p_{s}^*(A_s|\Fcal_{s-1})+p_{s}^{m'}(A_s|\Fcal_{s-1})}{2}\right)
        + \frac{1}{2} \DKL \left(p_{s}^*(A_s|\Fcal_{s-1}) , p_{s}^{m}(A_s|\Fcal_{s-1})\right) \right) \\
\label{majWt}
    & =: 9 W^{m'}_{n}.
\end{align}

Let $\beta = 9(n - N+1) y^2$, where $y >0$ is to be chosen later. Replacing $x$ by $x + \log(|\Mcal|)$ leads to
\begin{align*}
     \Pbb & \left(\frac{M^{m'}_{n+1}}{n-N+1} \geq 3 \sqrt{5\left( \frac{W^{m'}_{n}}{n-N+1} +y^2\right)\frac{x+\log(|\Mcal|)}{n-N+1}}+4\frac{x+\log(|\Mcal|)}{n-N+1} \, \bigg| \, \Fcal_{N-1}\right) \\
     & \qquad \qquad \leq \left(2 \log\left(y^{-2} + 1\right)+1 \right)e^{-(x+\log(|\Mcal|))}.
\end{align*}

Let $\kappa_1 \in (0,1/(8\sqrt{5})]$, then, using $2\sqrt{ab} \leq \kappa_1 a + \kappa_1^{-1} b$ and taking $y^2 = \frac{x+\log(|\Mcal|)}{(n-N+1) \log 2} \geq \frac{1}{n-N+1}$ since $x \geq 0$ and $|\Mcal| \geq 2$,
\begin{multline}\label{concentrationprefinalholdout}
\Pbb \left(\frac{M^{m'}_{n+1}}{n-N+1}
    \geq \frac{3\sqrt{5}}{2} \kappa_1 \frac{W^{m'}_{n}}{n-N+1}
        + \left( 4 + \frac{3\sqrt{5}}{2} \left(\frac{\kappa_1}{\log 2} + \kappa_1^{-1}\right) \right) \frac{x+\log(|\Mcal|)}{n-N+1}
    \, \bigg| \, \Fcal_{N-1} \right) \\
    \leq \left(2 \log\left(n-N + 2\right)+1 \right)e^{-(x+\log(|\Mcal|))}.
\end{multline}

By the union bound on all $m' \in \Mcal$, the previous inequality holds with probability at least $1-(2 \log(n-N+2) +1) e^{-x}$ for all $m' \in \Mcal$.
It holds in particular for $\hat{m}$.
Recall with~\eqref{preholdout} that,
\begin{multline}\label{preholdout2}
      \frac{1}{n-N+1} \sum_{s=N}^{n} \DKL\left(p_{s}^*(A_s|\Fcal_{s-1}), \frac{p_{s}^*(A_s|\Fcal_{s-1})+p_{s}^{\hat{m}}(A_s|\Fcal_{s-1})}{2}\right)- U_m \\
      \leq \frac{1}{2(n-N+1)} \sum_{s=N}^{n} \DKL\left(p_{s}^*(A_s|\Fcal_{s-1}), p_{s}^{m}(A_s|\Fcal_{s-1})\right)+\frac{M^{\hat{m}}_{n+1}}{n-N+1}.
\end{multline}

Plugging~\eqref{majWt} and~\eqref{concentrationprefinalholdout} in~\eqref{preholdout2} leads to, conditionally on $\Fcal_{N-1}$, with probability at least $1-(2 \log(n-N+2) +1) e^{-x}$\begin{multline*}
\frac{\left(1-C_{\kappa_1}\right)}{n-N+1} \sum_{s=N}^{n}  \DKL \left(p_{s}^*(A_s|\Fcal_{s-1}) , \frac{p_{s}^*(A_s|\Fcal_{s-1})+p_{s}^{\hat{m}}(A_s|\Fcal_{s-1})}{2}\right)- U_m\\
    \leq \frac{\left(1+C_{\kappa_1}\right)}{n-N+1} \sum_{s=N}^{n} \frac{1}{2} \DKL \left(p_{s}^*(A_s|\Fcal_{s-1}) , p_{s}^{m}(A_s|\Fcal_{s-1})\right) + C_{\kappa_1}' \frac{x+\log(|\Mcal|)}{n-N+1},
\end{multline*}
where
\begin{itemize}
    \item $C_{\kappa_1} \displaystyle = 8 \sqrt{5}\kappa_1$,
    \item $\displaystyle C_{\kappa_1}' = 4 + \frac{3\sqrt{5}}{2} \left(\frac{\kappa_1}{\log 2} + \kappa_1^{-1}\right) \leq 13 \kappa_1^{-1} = \frac{104 \sqrt{5}}{C_{\kappa_1}}$,
\end{itemize}

Integrating on $x \geq 0$ and noting that $\Ebb[U_m|\Fcal_{N-1}] = 0$ leads to, for all $m \in \Mcal$,
\begin{multline*}
\frac{\left(1-C_{\kappa_1}\right)}{n-N+1} \Ebb \left[\sum_{s=N}^{n}  \DKL \left(p_{s}^*(A_s|\Fcal_{s-1}) , \frac{p_{s}^*(A_s|\Fcal_{s-1})+p_{s}^{\hat{m}}(A_s|\Fcal_{s-1})}{2}\right)\Bigg| \Fcal_{N-1}\right] \\
    \leq \frac{\left(1+C_{\kappa_1}\right)}{n-N+1} \Ebb\left[\frac{1}{2}\sum_{s=N}^{n} \DKL \left(p_{s}^*(A_s|\Fcal_{s-1}) , p_{s}^{m}(A_s|\Fcal_{s-1})\right)\Bigg| \Fcal_{N-1}\right]\\
    + \frac{104 \sqrt{5}}{C_{\kappa_1}} \frac{2 \log(n-N+2) + 1 + \log(|\Mcal|)}{n-N+1}.
\end{multline*}
Pick any $\kappa_1$ so that $C_{\kappa_1} \in (0,1]$. And finally, with \citep[Lemma 7.23]{massart2007concentration}, 
\begin{equation*}
    (2 \log 2 - 1) \DHL^2(p_{s}^*(A_s|\Fcal_{s-1}), p_{s}^{\hat{m}}(A_s|\Fcal_{s-1})) \leq \DKL \left(p_{s}^*(A_s|\Fcal_{s-1}) , \frac{p_{s}^*(A_s|\Fcal_{s-1})+p_{s}^{\hat{m}}(A_s|\Fcal_{s-1})}{2}\right).
\end{equation*}
\end{proof}

\subsection{Proof of Section~\ref{section_penalizedmle}}\label{sectionproofgeneralmodelselection}

\begin{proof}[Proof of Proposition~\ref{prop_atail}]\label{proof_prop_atail}
    The proof is straightforward with the definition of $p_{\theta^m,t}^m$ in~\eqref{eq_probmodele}. Let $\theta^m,\delta^m \in \Theta^{m}$, $t \leq T_\varepsilon$, $x \in \Xcal$, and $k \in [K]$. 
    \begin{equation*}
        p_{\theta^m,t}^m(k | x) = \sum_{C \in \Pcal_m} \pi_{C,T_t^C}^{\theta_C}(k) \one_{x \in C} \geq \sum_{C \in \Pcal_m} \varepsilon \one_{x \in C} = \varepsilon.
    \end{equation*}
    For the second part of the proof, it holds that, almost surely, for all $t \leq T_\varepsilon$
    \begin{multline*}
        \left|\log\left(\frac{p_{\delta^m,t}^m(k | x)}{p_{\theta^m,t}^m(k | x)}\right)\right| = \sum_{C \in \Pcal_m} \left|\log\left(\frac{\pi_{C,T_t^C}^{\delta_C}(k)}{\pi_{C,T_t^C}^{\theta_C}(k)}\right)\right| \one_{x \in C}
        \leq L_\varepsilon \sum_{C \in \Pcal_m} \| \delta_C- \theta_C\|_{2}\, \one_{x \in C} \leq L_\varepsilon \sup_{C \in \Pcal_m} \| \delta_C- \theta_C\|_{2}. 
    \end{multline*}
\end{proof}

\begin{proof}[Proof of Theorem~\ref{generalmodelselectiontheorem}]
Our goal is to apply \citep{aubert:generalpmle}.

Assumption 1 of \citep{aubert:generalpmle} is satisfied for $n = T_\varepsilon$ since with Proposition~\ref{prop_atail}, there exists $\varepsilon > 0$ such that a.s., for all $t \in [T_\varepsilon]$, for all $k \in [K]$, 
$p_t^\star(k | X_t) \in [\varepsilon, 1]$ and for all $m \in \Mcal$ and all $\theta^m \in \Theta^{D_m}$, 
$p^m_{\theta^m,t}(k | X_t) \in [\varepsilon, 1]$.

Assumption 2 of \citep{aubert:generalpmle} is satisfied since with Proposition~\ref{prop_atail},
there exists a positive constants $L_ \varepsilon$ such that a.s., for all $t \in [T_\varepsilon]$,  for all $m \in \Mcal$ and all $\delta^m, \theta^m \in \Theta^{D_m}$,
\begin{equation*}
\sup_{k \in [K]} \left| \log \left(\frac{p^m_{\delta^m,t}(k | X_t)}{p^m_{\theta^m,t}(k | X_t)} \right)\right|
    \leq L_\varepsilon \sup_{C \in \Pcal_m} \| \delta_C- \theta_C\|_{2}
\end{equation*}
and by Assumption, for all $\theta^m, \delta^m \in \Theta^{D_m}$
\begin{equation*}
    \sup_{C \in \Pcal_m}\|\delta_C- \theta_C\|_{2} \leq \sqrt{d}(R-r).
\end{equation*}
Note in particular that the Lipschitz constant in Proposition~\ref{prop_atail} does not depend on $m$.

Assumptions 3 and 4 in \citep{aubert:generalpmle} are always satisfied because the set of actions $[K]$ is finite.

Setting $A_\varepsilon = L_\varepsilon \sqrt{d}(R-r) + 2\log(\varepsilon^{-1})$, Corollary 2 in \citep{aubert:generalpmle} simply reads as follows. There exist positive numerical constants $C$ and $C'$ such that the following holds. Assume that
  \begin{equation*}
\Sigma_\varepsilon = \log(A_\varepsilon)\sum_{m \in \Mcal} e^{-D_m} < +\infty .
\end{equation*}
Let $\kappa \in (0,1]$.
If for all $m \in \Mcal$,
\begin{equation*}
    \pen(m) \geq \frac{C}{\kappa} A_\varepsilon^2 \log(\varepsilon^{-1}) \log(T_\varepsilon A_\varepsilon)^2 \frac{D_m}{T_\varepsilon},
\end{equation*}
then,
    \begin{multline*}
        \frac{1 - \kappa}{T_\varepsilon} 
\sum_{t=1}^{T_\varepsilon} \Ebb\left[\DKL\left(p_t^\star(A_t | X_t), p_{\hat{\theta}^{\hat{m}},t}^{\hat{m}}(A_t | X_t)\right)\right]
    \\
    \leq \Ebb \left[\inf_{m \in \Mcal} \left( (1 + \kappa) \inf_{\theta \in \Theta^m} \frac{1}{T_\varepsilon} \sum_{t=1}^{T_\varepsilon}\DKL\left(p_t^\star(A_t | X_t), p_{\theta^m,t}^m(A_t | X_t)\right)
        + 2 \pen(m)
    \right)\right] \\
        + \frac{36 C'}{\kappa} 
            A_\varepsilon  \Sigma_\varepsilon \log(\varepsilon^{-1})
            \frac{\log(T_\varepsilon)}{T_\varepsilon}.
    \end{multline*}
To conclude take for instance $\kappa = 1/2$, hence the result.
\end{proof}

\subsection{Proof of Section~\ref{exampleExp3}}

\begin{proof}[Proof of Proposition~\ref{hypotheseexistenceupsilonExp3}]
Let $\theta_C \in \Theta_C$. Write $\theta_{C,n} =(\eta_{C,n},\gamma_{C,n}) = \theta_C/\sqrt{n} \in \Theta$. To ease the notations in the proof, we remove the $C$ from the notations. $\theta_C$ becomes $\theta$ and $\theta_{C,n}$ becomes $\theta_n$. In the same way, $\theta_n=(\eta_n,\gamma_n)$ now.

Let $t \in F_n(C)$. Then,
\begin{align*}
\pi_{C,T_t^C+1}^{\theta}(k)
    &= \frac{\pi_{C,T_t^C}^{\theta}(k) e^{-\eta_n g_{k,t} / (\gamma_n + \pi_{C,T_t^C}^{\theta}(k))}\one_{A_t = k}}{(1-\pi_{C,T_t^C}^{\theta}(k))+\pi_{C,T_t^C}^{\theta}(k) e^{-\eta_n g_{k,t} / (\gamma_n + \pi_{C,T_t^C}^{\theta}(k))}}\\
    & \qquad \qquad \qquad+ \sum_{\underset{j \neq k}{j=1}}^K \frac{\pi_{C,T_t^C}^{\theta}(k) \one_{A_t=j}}{(1-\pi_{C,T_t^C}^{\theta}(j)) + \pi_{C,T_t^C}^{\theta}(j)e^{-\eta_n g_{j,t} / (\gamma_n + \pi_{C,T_t^C}^{\theta}(j))}}.
\end{align*}
For any $q \in [0,1]$, since $g_{k,t} \in [0,1]$, $1-q + q e^{-\eta g_{k,t}/(q+\gamma)} \leq 1$. Therefore, 
\begin{equation*}
\pi_{C,T_t^C+1}^{\theta}(k)
    \geq \pi_{C,T_t^C}^{\theta}(k) e^{-\eta_n g_{k,t} / (\gamma_n + \pi_{C,T_t^C}^{\theta}(k))} \one_{A_t = k}
    + \sum_{\underset{j \neq k}{j=1}}^K \pi_{C,T_t^C}^{\theta}(k) \one_{A_t=j}.
\end{equation*}
Since $e^{-\eta_n g_{k,t} / (\gamma_n + \pi_{C,T_t^C}^{\theta}(k))} \leq 1$, 
\begin{align*}
&\pi_{C,T_t^C+1}^{\theta}(k)
\\
    &\geq \pi_{C,T_t^C}^{\theta}(k) e^{-\eta_n g_{k,t} / (\gamma_n + \pi_{C,T_t^C}^{\theta}(k))} \one_{A_t = k}
    + \sum_{\underset{j \neq k}{j=1}}^K \pi_{C,T_t^C}^{\theta}(k) e^{-\eta_n g_{k,t} / (\gamma_n + \pi_{C,T_t^C}^{\theta}(k))} \one_{A_t=j}
    \\
    &= \pi_{C,T_t^C}^{\theta}(k) e^{-\eta_n g_{k,t} / (\gamma_n + \pi_{C,T_t^C}^{\theta}(k))} \geq \pi_{C,T_t^C}^{\theta}(k) e^{-\eta_n / (\gamma_n + \pi_{C,T_t^C}^{\theta}(k))}
\end{align*}
since $g_{k,t} \in [0,1]$. Then,
\begin{align*}
\pi_{C,T_t^C+1}^{\theta}(k) \geq \pi_{C,T_t^C}^{\theta}(k) \left(1-\frac{\eta_n g_{k,t}}{\gamma_n + \pi_{C,T_t^C}^{\theta}(k)}\right) \geq \pi_{C,T_t^C}^{\theta}(k)- \eta_n .
\end{align*}
Summing for all $s \in F_t(C)$, since $\pi_{k,1}^{\theta}= \frac{1}{K}$,
\begin{equation*}
\pi_{C,T_t^C}^{\theta}(k) \geq \frac{1}{K} - \eta_n T_t^C.
\end{equation*}
Note that $T_t^C \leq t \leq T_\varepsilon$. Since, $\displaystyle T_\varepsilon \leq \left \lfloor \left(\frac{1}{K}- \varepsilon\right) \frac{\sqrt{n}}{R}\right \rfloor$, for all $t \leq T_\varepsilon$ and $1 \leq k \leq K$,
\begin{equation*}
\varepsilon
    \leq \frac{1}{K} - \frac{R}{\sqrt{n}} T_\varepsilon
    \leq \frac{1}{K}- \eta_n T_\varepsilon
    \leq \frac{1}{K}- \eta_n t \leq  \pi_{C,T_t^C}^{\theta}(k).
\end{equation*}
For the second part of the proof, let $\theta=(\eta,\gamma), \theta' = (\eta',\gamma') \in \Theta_C$. For $t \geq 2$  Let $h_{j,t}^{\theta} = \eta_n \sum_{s \in F_t(C)} \hat{g}_{j,s}^{\theta}$. Then $\pi_{C,T_t^C}^{\theta} = \softmax(h_{\cdot,t}^{\theta}).$ The function $\softmax$ is $1$-Lipschitz with respect to the $\|\cdot\|_2$-norm in $\Rbb^K$ (see \citep{gao2017properties} for a proof). Therefore,
\begin{equation*}
    \|\pi_{C,T_t^C}^{\theta}-\pi_{C,T_t^C}^{\theta'}\|_{2} \leq \|h_{\cdot,t}^{\theta}-h_{\cdot,t}^{\theta'}\|_{2}.
\end{equation*}
Since $g_{j,s} \in [0,1]$, by the triangle inequality
\begin{equation*}
    \|\pi_{C,T_t^C}^{\theta}-\pi_{C,T_t^C}^{\theta'}\|_{2} \leq \sum_{s \in F_t(C)}\left\| \left(\frac{\eta_n}{\gamma_n + \pi_{C,T_s^C}^{\theta}(.)}-\frac{\eta_n'}{\gamma_n' + \pi_{C,T_s^C}^{\theta'}(.)}\right) \one_{A_s = \cdot} \right\|_{2}.
\end{equation*}
Again, using the triangle inequality,
\begin{align}\label{pregronwall0}
    \nonumber \|\pi_{C,T_t^C}^{\theta}-\pi_{C,T_t^C}^{\theta'}\|_{2} &\leq \sum_{s\in F_t(C)}\left\| \left(\frac{\eta_n}{\gamma_n + \pi_{C,T_s^C}^{\theta}}-\frac{\eta_n'}{\gamma_n' + \pi_{C,T_s^C}^{\theta}}\right) \one_{A_s = \cdot} \right\|_{2}\\
    & \qquad \qquad +\sum_{s\in F_t(C)}\left\| \left(\frac{\eta_n'}{\gamma_n' + \pi_{C,T_s^C}^{\theta}}-\frac{\eta_n'}{\gamma_n' + \pi_{C,T_s^C}^{\theta'}}\right) \one_{A_s = \cdot} \right\|_{2}.
\end{align}
For $1 \geq q \geq \varepsilon$, let $ \displaystyle f : (x_1,x_2) \in [0,R_{n}]\times[0,R_{n}] \mapsto \frac{x_1}{x_2 + q}$ where $R_{n} = \frac{R}{\sqrt{n}} $. The function $f$ is continuously differentiable, and
\begin{align*}
    \nabla f = \frac{1}{(x_2+q)^2} & \begin{pmatrix}
           x_{2}+q \\
           -x_{1}
    \end{pmatrix}.
  \end{align*}
The $\ell^2$-norm of the gradient can be upper bounded by
\begin{equation*}
    \|\nabla f\|_2 \leq \frac{1}{\varepsilon^2} \sqrt{R_n^2 + \varepsilon^2}=:c_\varepsilon
\end{equation*}
By the mean value theorem, for all $k \in [K]$
\begin{equation*}
    \left|\frac{\eta_n}{\gamma_n + \pi_{C,T_s^C}^{\theta}(k)}-\frac{\eta_n'}{\gamma_n' + \pi_{C,T_s^C}^{\theta}(k)}\right| \leq c_\varepsilon  \|\theta_n- \theta_n'\|_2.
\end{equation*}
As a result,
\begin{align*}
    \left\| \left(\frac{\eta_n}{\gamma_n + \pi_{C,T_s^C}^{\theta}}-\frac{\eta_n'}{\gamma_n' + \pi_{C,T_s^C}^{\theta}}\right) \one_{A_s = \cdot} \right\|_{2}^2 &= \sum_{k=1}^K \left(\frac{\eta_n}{\gamma_n + \pi_{C,T_s^C}^{\theta}(k)}-\frac{\eta_n'}{\gamma_n' + \pi_{C,T_s^C}^{\theta}(k)}\right)^2 \one_{A_s = k}\\
    & \leq c_\varepsilon^2  \|\theta_n - \theta_n'\|_2^2 \sum_{k=1}^K \one_{A_s = k} \\
    & = c_\varepsilon^2  \|\theta_n - \theta_n'\|_2^2
\end{align*}
Therefore,
\begin{align}\label{pregronwall1}
   \nonumber & \sum_{s\in F_t(C)}\left\| \left(\frac{\eta_n}{\gamma_n + \pi_{C,T_s^C}^{\theta}}-\frac{\eta_n'}{\gamma_n' + \pi_{C,T_s^C}^{\theta}}\right) \one_{A_s = \cdot} \right\|_{2} \\
   & \qquad \qquad \qquad \leq \sum_{s\in F_t(C)} c_{\varepsilon} \|\theta_n - \theta_n'\|_2 = T_t^C c_{\varepsilon} \|\theta_n - \theta_n'\|_2\leq T_\varepsilon c_{\varepsilon} \|\theta_n - \theta_n'\|_2.
\end{align}
For $(\eta,\gamma) \in [0,R_{n}]\times[0,R_{n}]$, let $\displaystyle g : q \in [\varepsilon,1] \mapsto \frac{\eta}{\gamma + q}$. The function $g$ is continuously differentiable, and
\begin{equation*}
    0 \leq f'(q) = \frac{\eta}{(\gamma+q)^2} \leq \frac{R_{n}}{\varepsilon^2}.
\end{equation*}
By the mean value theorem, for all $k \in [K]$,
\begin{equation*}
    \left|\frac{\eta_n'}{\gamma_n' + \pi_{C,T_s^C}^{\theta}(k)}-\frac{\eta_n'}{\gamma_n' + \pi_{C,T_s^C}^{\theta'}}(k)\right| \leq \frac{R_{n}}{\varepsilon^2} \left|\pi_{C,T_s^C}^{\theta}(k)-\pi_{C,T_s^C}^{\theta'}(k)\right|.
\end{equation*}
Therefore,
\begin{align*}
    \left\| \left(\frac{\eta_n'}{\gamma_n' + \pi_{C,T_s^C}^{\theta}}-\frac{\eta_n'}{\gamma_n' + \pi_{C,T_s^C}^{\theta'}}\right) \one_{A_s = \cdot} \right\|_{2}^2 &= \sum_{k=1}^K  \left(\frac{\eta_n'}{\gamma_n' + \pi_{C,T_s^C}^{\theta}(k)}-\frac{\eta_n'}{\gamma_n' + \pi_{C,T_s^C}^{\theta'}}\right)^2 \one_{A_s = k} \\
    & \leq \frac{R_{n}^2}{\varepsilon^4} \sum_{k=1}^K \left|\pi_{C,T_s^C}^{\theta}(k)-\pi_{C,T_s^C}^{\theta'}(k)\right|^2 \one_{A_s = k}\\
    & \leq \frac{R_{n}^2}{\varepsilon^4}\left(\left\|\pi_{C,T_s^C}^{\theta}-\pi_{C,T_s^C}^{\theta'}\right\|_{2}\right)^2.
\end{align*}
Thus,
\begin{equation}\label{pregronwall2}
     \sum_{s\in F_t(C)}\left\| \left(\frac{\eta_n'}{\gamma_n' + \pi_{C,T_s^C}^{\theta}}-\frac{\eta_n'}{\gamma_n' + \pi_{C,T_s^C}^{\theta'}}\right) \one_{A_s = \cdot} \right\|_{2} \leq \frac{R_{n}}{\varepsilon^2} \sum_{s\in F_t(C)}\left\|\pi_{C,T_s^C}^{\theta}-\pi_{C,T_s^C}^{\theta'}\right\|_{2}
\end{equation}
Plugging Equations~\eqref{pregronwall1} and~\eqref{pregronwall2} in Equation~\eqref{pregronwall0}
\begin{equation*}
     \|\pi_{C,T_t^C}^{\theta}-\pi_{C,T_t^C}^{\theta'}\|_{2} \leq T_\varepsilon c_\varepsilon \|\theta_n - \theta_n'\|_2 + \frac{R_{n}}{\varepsilon^2}\sum_{s \in F_t(C)}\left\|\pi_{C,T_s^C}^{\theta}-\pi_{C,T_s^C}^{\theta'}\right\|_{2},
\end{equation*}
Using the discrete Gronwall Lemma \citep{CLARK1987279} leads to, for all $t \leq T_\varepsilon$,
\begin{align*}
     \|\pi_{C,T_t^C}^{\theta}-\pi_{C,T_t^C}^{\theta'}\|_{2} & \leq T_\varepsilon c_\varepsilon \|\theta_n - \theta_n'\|_2 \prod_{s \in F_t(C)}\left(1+\frac{R_{n}}{\varepsilon^2}   \right) \\
     &\leq T_\varepsilon c_\varepsilon \|\theta_n - \theta_n'\|_2 \exp\left(\frac{R_{n} T_\varepsilon}{\varepsilon^2}\right).
\end{align*}
But, since $\frac{1}{K}-\varepsilon \leq 1$,
\begin{equation*}
    T_\varepsilon \leq \left(\frac{1}{K}-\varepsilon\right) \frac{\sqrt{n}}{R} \leq \frac{\sqrt{n}}{R}.
\end{equation*}
Therefore, $R_{n}T_\varepsilon \leq 1$ and for $t \leq T_\varepsilon$,
\begin{align*}
     \|\pi_{C,T_t^C}^{\theta}-\pi_{C,T_t^C}^{\theta'}\|_{2} & \leq \frac{c_\varepsilon}{R}e^{1/\varepsilon^2}\|\theta - \theta'\|_2.
\end{align*}
To conclude note that $\log$ is $1/\varepsilon$-Lipschitz on $[\varepsilon,1]$ and that 
\begin{equation*}
    \sup_{k \in [K]} \left|\log\left(\frac{\pi_{C,T_t^C}^{\delta_C}(k)}{\pi_{C,T_t^C}^{\theta_C}(k)}\right)\right| \leq \frac{1}{\varepsilon}  \|\pi_{C,T_t^C}^{\theta}-\pi_{C,T_t^C}^{\theta'}\|_{2}.
\end{equation*}
\end{proof}

\subsection{Proof of Section~\ref{examplegradientbandits}}

\begin{proof}[Proof of Proposition~\ref{existencehypotheseupsilongradientbrandits}] We take the same notations as the previous section. The updated probability can be written
\begin{align*}
\pi_{C,T_t^C+1}^{\theta}(k)
    = \frac{\pi_{C,T_t^C}^{\theta}(k) e^{\theta_n \left(\one_{A_t = k}-\pi_{C,T_t^C}^{\theta}(k)\right) g_{A_t,t}}}{\pi_{C,T_t^C}^{\theta}(A_t)e^{\theta_n (1-\pi_{C,T_t^C}^{\theta}(A_t))g_{A_t,t}}+ \sum_{j \neq A_t} \pi_{C,T_t^C}^{\theta}(j) e^{- \theta_n \pi_{C,T_t^C}^{\theta}(j) g_{A_t,t}}}.
\end{align*}
Therefore,
\begin{align*}
\pi_{C,T_t^C+1}^{\theta}(k)
    & \geq \frac{\pi_{C,T_t^C}^{\theta}(k) e^{\theta_n \left(\one_{A_t = k}-\pi_{C,T_t^C}^{\theta}(k)\right) g_{A_t,t}}}{\pi_{C,T_t^C}^{\theta}(A_t)e^{\theta_n (1-\pi_{C,T_t^C}^{\theta}(A_t))g_{A_t,t}}+ \sum_{j \neq A_t} \pi_{C,T_t^C}^{\theta}(j)}\\
    & \geq \frac{\pi_{C,T_t^C}^{\theta}(k) e^{-\theta_n}}{\pi_{C,T_t^C}^{\theta}(A_t)e^{\theta_n}+ 1-\pi_{C,T_t^C}^{\theta}(A_t)} \\
    & \geq  \pi_{C,T_t^C}^{\theta}(k) e^{-2\theta_n}
\end{align*}
where
\begin{itemize}
    \item the first inequality holds because $e^{- \theta_n \pi_{C,T_t^C}^{\theta}(j) g_{A_t,t}} \leq 1$,
    \item the second inequality holds because $g_{j,t} \in (0,1)$, for $j\in [K]$,
    \item the last inequality holds because $\pi_{C,T_t^C}^{\theta}(A_t)e^{\theta_n}+ 1-\pi_{C,T_t^C}^{\theta}(A_t) \leq e^{\theta_n}$.
\end{itemize}
Thus, for all $t \leq T_\varepsilon$, 
\begin{align*}
   \pi_{C,T_t^C}^{\theta}(k) \geq \frac{1}{K}e^{-2\theta_n T_t^C}. 
\end{align*}
Since $T_t^C \leq t$ and since by definition, $T_\varepsilon \leq \log\left(\sqrt{\frac{1}{K\varepsilon}}\right)\frac{\sqrt{n}}{R} $, it holds that for all $t \leq T_\varepsilon$,
\begin{align*}
    \pi_{C,T_t^C}^{\theta}(k) \geq \frac{1}{K}e^{-2\theta_n t} \geq \frac{1}{K}e^{-2 R_n t} \geq \frac{1}{K}e^{-2 \frac{R}{\sqrt{n}} \log\left(\sqrt{\frac{1}{K\varepsilon}}\right)\frac{\sqrt{n}}{R}} \geq \frac{1}{K} e^{-\log\left(\frac{1}{K\varepsilon}\right)} \geq \varepsilon.
\end{align*}
For the second part of the proof, for $t \geq 2$ and $j \in [K]$, let $h_{j,T_t^C}^{\theta} = \theta_n \sum_{s\in F_t(C)} \hat{g}_{j,s}^{\theta}$. Then $\pi_{C,T_t^C}^{\theta} = \softmax(h_{\cdot,t}^{\theta}).$ The function $\softmax$ is $1$-Lipschitz with respect to the $\|\cdot\|_2$-norm in $\Rbb^K$ (see \citep{gao2017properties} for a proof). Therefore,
\begin{equation*}
    \|\pi_{C,T_t^C}^{\delta}-\pi_{C,T_t^C}^{\theta}\|_{2} \leq \|h_{\cdot,t}^{\delta}-h_{\cdot,t}^{\theta}\|_{2}.
\end{equation*}
Then,
\begin{align*}
   \|h_{\cdot,t}^{\delta}-h_{\cdot,t}^{\theta}\|_{2} &\leq |\delta_n -\theta_n| \sum_{s\in F_t(C)} \|\hat{g}_{j,s}^{\delta} \|_2 + \theta_n \sum_{s\in F_t(C)} g_{A_s,s}\|\pi_{C,T_s^C}^{\delta}-\pi_{C,T_s^C}^{\theta} \|_2\\
    & \leq \sqrt{2}T_t^C|\delta_n - \theta_n|+ \theta_n \sum_{s\in F_t(C)}\|\pi_{C,T_s^C}^{\delta}-\pi_{C,T_s^C}^{\theta}\|_{2}\\
    & \leq \sqrt{2}T_\varepsilon|\delta_n - \theta_n|+ \theta_n \sum_{s\in F_t(C)}\|\pi_{C,T_s^C}^{\delta}-\pi_{C,T_s^C}^{\theta}\|_{2}.
\end{align*}
where 
\begin{itemize}
    \item the first inequality holds because of the triangle inequality,
    \item the second inequality holds because for all $j \in [K]$, $g_{j,s} \in [0,1]$ and
\begin{equation*}
    \|\hat{g}_{j,s}^{\delta} \|_2^2 = (1-\pi_{C,T_s^C}^{\delta}(A_s))^2 + \sum_{j \neq A_s} (\pi_{C,T_s^C}^{\delta}(j))^2 \leq 2,
\end{equation*}
    \item the last inequality holds because $T_t^C \leq T_\varepsilon$.
\end{itemize}
By the discrete Gronwall Lemma \citep{CLARK1987279}, for all $t \leq T_\varepsilon$
\begin{align*}
    \|\pi_{C,T_t^C}^{\delta}-\pi_{C,T_t^C}^{\theta}\|_{2} & \leq \sqrt{2}|\delta_n - \theta_n|T_\varepsilon \prod_{s\in F_t(C)}\left(1 + \theta_n\right) \leq \sqrt{2}|\delta_n - \theta_n|T_\varepsilon e^{\theta_n T_\varepsilon}.
\end{align*}
Since $T_\varepsilon \leq \left \lfloor \log\left(\sqrt{\frac{1}{K\varepsilon}}\right)\frac{\sqrt{n}}{R} \right \rfloor$, $\theta_n T_\varepsilon \leq R_n T_\varepsilon \leq \log\left(\sqrt{\frac{1}{K\varepsilon}}\right)$, therefore,
\begin{align*}
    \|\pi_{C,T_t^C}^{\delta}-\pi_{C,T_t^C}^{\theta}\|_{2} \leq \frac{\sqrt{2}}{R}\frac{\log\left(\sqrt{\frac{1}{K\varepsilon}}\right)}{\sqrt{K \varepsilon}}|\delta - \theta|.
\end{align*}
Finally, $\log$ is $\frac{1}{\varepsilon}$-Lipschitz on $[\varepsilon,1]$. Thus, for all $k \in [K]$, $t \leq T_\varepsilon$,
\begin{equation*}
      \left|\log\left( \frac{\pi_{C,T_t^C}^{\delta}(k)}{\pi_{C,T_t^C}^{\theta}(k)}\right) \right| \leq \frac{1}{\varepsilon} |\pi_{C,T_t^C}^{\delta}(k)-\pi_{C,T_t^C}^{\theta}(k)| \leq \frac{\sqrt{2}}{R \varepsilon}\frac{\log\left(\sqrt{\frac{1}{K\varepsilon}}\right)}{\sqrt{K \varepsilon}}|\delta - \theta|.
\end{equation*}
\end{proof}

\subsection{Proof of Section~\ref{section_metalearning}}

Let us recall that $|F| = \max_{m \in \Mcal} |E_m|$. For this Section, we drop the dependence $m$ of the model and simply write $E$ and $\theta$ generic set of policies and parameter in $[r,R]$. 

\begin{proof}[Proof of Proposition~\ref{prop_expert_minorant}]\label{proof_prop_expert_minorant}
For any $\theta \in [r,R]$, write $\theta_n = \theta/\sqrt{n}$. Assume that Assumption~\ref{hyp_expert_model} holds. Let's write $q_{E,t+1}^\theta$ as
\begin{equation*}
    q_{E,t+1}^{\theta}(j) = \frac{q_{E,t}^{\theta}(j) e^{-\theta_n \hat{y}_{j,t}^\theta}}{\sum_{i \in E} q_{E,t}^{\theta}(i) e^{-\theta_n \hat{y}_{i,t}^\theta}}
\end{equation*}
Since $q_{E,t}^\theta$ is a probability distribution over the experts,
\begin{equation*}
    \sum_{i \in E} q_{E,t}^{\theta}(i) e^{-\theta_n \hat{y}_{i,t}^\theta} \leq \sum_{i \in E} q_{E,t}^{\theta}(i) = 1.
\end{equation*}
Therefore, $ q_{E,t+1}^{\theta}(j) \geq q_{E,t}^{\theta}(j) e^{-\theta_n \hat{y}_{j,t}^\theta}$. By definition, 
\begin{equation*}
     \hat{y}_{j,t}^\theta = \sum_{k = 1}^K \xi_{j,t}(k) \frac{g_{k,t}}{\pi_{E,t}^{\theta}(k)} \one_{A_t = k} = \xi_{j,t}(A_t) \frac{g_{A_t,t}}{\pi_{E,t}^{\theta}(A_t)}.
\end{equation*}
Using that $e^{-x} \geq 1-x$ for any $x \geq 0$, leads to
\begin{equation*}
     q_{E,t+1}^{\theta}(j) \geq q_{E,t}^{\theta}(j) \left(1-\theta_n \xi_{j,t}(A_t) \frac{g_{A_t,t}}{\pi_{E,t}^{\theta}(A_t)} \right).
\end{equation*}
Since $g_{A_t,t} \in [0,1]$ and $q_{E,t}^{\theta}(j) \xi_{j,t}(A_t) \leq  \pi_{E,t}^{\theta}(A_t)$,
\begin{equation*}
    q_{E,t+1}^{\theta}(j) \geq q_{E,t}^{\theta}(j) - \theta_n.
\end{equation*}
Summing for all $s$ from $1$ to $t$,
\begin{equation*}
    q_{E,t}^{\theta}(j) \geq \frac{1}{|E|} - \theta_n t.
\end{equation*}
Since
\begin{equation*}
    T_\varepsilon = \left \lfloor \left( \frac{1}{|F|}-\frac{\varepsilon}{\rho}\right) \frac{\sqrt{n}}{R} \right \rfloor \wedge n \qquad \text{and} \qquad |F| \geq |E|,
\end{equation*}
it holds that,
\begin{equation*}
    T_\varepsilon \leq \left \lfloor \left( \frac{1}{|E|}-\frac{\varepsilon}{\rho}\right) \frac{\sqrt{n}}{R} \right \rfloor \wedge n.
\end{equation*}
Therefore, for all $t \leq T_\varepsilon$,
\begin{equation*}
    q_{E,t}^{\theta}(j) \geq \frac{1}{|E|}-\frac{R}{\sqrt{n}} t \geq \frac{1}{|E|}-\frac{R}{\sqrt{n}} \left( \frac{1}{|E|}-\frac{\varepsilon}{\rho}\right)\frac{\sqrt{n}}{R} = \frac{\varepsilon}{\rho}.
\end{equation*}
Finally, for all $t \leq T_\varepsilon$, for all $k \in [K]$,
\begin{equation*}
    \pi_{E,t}^{\theta}(k) = \sum_{j \in E} q_{E,t}^{\theta}(j) \xi_{j,t}(k) \geq \frac{\varepsilon}{\rho} \sum_{j \in E} \xi_{j,t}(k) = \varepsilon.
\end{equation*}
For the second part of the proof, let $\eta, \delta \in [r,R]$, and write $\eta_n = \eta/\sqrt{n}$ and likewise for $\delta_n$. For $t \geq 2$, let $g_{j,t}^\eta = \eta_n \sum_{s=1}^{t-1} \hat{y}_{j,s}^{\eta}$. Then, $q_{E,t}^\eta = \softmax(g_t^\eta)$ where $g_t^\eta = (g_{j,t}^\eta)_{j \in E}$. Since the function softmax is $1$-Lipschitz with respect to the $\|\cdot\|_2$-norm in $\Rbb^{|E|}$,\begin{equation*}
    \|q_{E,t}^\eta-q_{E,t}^\delta\|_2 \leq \|g_t^\eta-g_t^\delta\|_2.
\end{equation*}
Therefore, by the triangle inequality,
\begin{align}\label{eq_gronwallexperts0}
   \nonumber \|q_{E,t}^\eta-q_{E,t}^\delta\|_2 &\leq \sum_{s=1}^{t-1} \|\eta_n \hat{y}_{j,s}^{\eta} -
\delta_n \hat{y}_{j,s}^{\delta} \|_2\\
&\leq \sum_{s=1}^{t-1} \left(|\eta_n -\delta_n| \| \hat{y}_{j,s}^{\eta}\|_2 + \delta_n \|\hat{y}_{j,s}^{\eta}-\hat{y}_{j,s}^{\delta} \|_2\right).
\end{align}
Since $\xi_{j,t}$ is a probability distribution,
\begin{multline*}
   \| \hat{y}_{j,t}^\eta\|_2^2 =  \sum_{j \in E}(\hat{y}_{j,t}^\eta)^2 = \sum_{j \in E} \left(\sum_{k = 1}^K \xi_{j,t}(k) \frac{g_{k,t}}{\pi_{E,t}^{\eta}(k)} \one_{A_t = k}\right)^2  \\
   = \sum_{j \in E} \left(\xi_{j,t}(A_t) \frac{g_{A_t,t}}{\pi_{E,t}^{\eta}(A_t)}\right)^2 \leq |E| \left(\frac{g_{A_t,t}}{\pi_{E,t}^{\eta}(A_t)}\right)^2.
\end{multline*}
Since $g_{A_t,t}^\eta \in [0,1]$ and $\pi_{E,t}^{\eta}(A_t) \geq \varepsilon$ for all $t \leq T_\varepsilon$,
\begin{equation}\label{eq_gronwallexperts1}
     \| \hat{y}_{j,t}^\eta\|_2 \leq \frac{\sqrt{|E|}}{\varepsilon}.
\end{equation}
Similarly,
\begin{align*}
  \|\hat{y}_{j,s}^{\eta}-\hat{y}_{j,s}^{\delta} \|_2^2 &= \sum_{j \in E} \left(\sum_{k = 1}^K \xi_{j,t}(k) g_{k,t} \left(\frac{1}{\pi_{E,t}^{\eta}(k)}-\frac{1}{\pi_{E,t}^{\delta}(k)}\right) \one_{A_t = k}\right)^2\\
  &\leq \sum_{j \in E} \left(\sum_{k = 1}^K \xi_{j,t}(k) \left(\frac{1}{\pi_{E,t}^{\eta}(k)}-\frac{1}{\pi_{E,t}^{\delta}(k)}\right) \one_{A_t = k}\right)^2.
\end{align*}
Thus, for all $t \leq T_\varepsilon$, 
\begin{equation*}
    \|\hat{y}_{j,t}^{\eta}-\hat{y}_{j,t}^{\delta} \|_2^2 \leq \frac{1}{\varepsilon^4} \sum_{j \in E} \left(\sum_{k = 1}^K \xi_{j,t}(k) \left(\pi_{E,t}^{\eta}(k)-\pi_{E,t}^{\delta}(k)\right) \one_{A_t = k}\right)^2.
\end{equation*}
Since $\xi_{j,t}(k)\leq 1$, 
\begin{equation}\label{eq_gronwallexperts2}
    \|\hat{y}_{j,t}^{\eta}-\hat{y}_{j,t}^{\delta} \|_2 \leq \frac{\sqrt{|E|}}{\varepsilon^2} \|\pi_{E,t}^\eta-\pi_{E,t}^\delta\|_2.
\end{equation}
Injecting~\eqref{eq_gronwallexperts1} and~\eqref{eq_gronwallexperts2} in~\eqref{eq_gronwallexperts0} leads to
\begin{equation*}
    \|q_{E,t}^\eta-q_{E,t}^\delta\|_2 \leq \frac{\sqrt{|E|}}{\varepsilon}|\eta_n- \delta_n|(t-1) + \delta_n \frac{\sqrt{|E|}}{\varepsilon^2} \sum_{s=1}^{t-1} \|\pi_{E,s}^\eta - \pi_{E,s}^\delta\|_2.
\end{equation*}
On the other hand,
\begin{align*}
    \|\pi_{E,t}^\eta - \pi_{E,t}^\delta\|_2^2= \sum_{k=1}^K (\pi_{k,t}^\eta - \pi_{k,t}^\delta)^2 &= \sum_{k=1}^K \left( \sum_{j \in E} \xi_{j,t}(k)(q_{E,t}^\eta(j) - q_{E,t}^\delta(j))\right)^2\\
    &\leq \sum_{k=1}^K \sum_{j \in E} \xi_{j,t}(k)^2 \sum_{j \in E} (q_{E,t}^\eta(j) - q_{E,t}^\delta(j))^2\\
    &\leq |E| \|q_{E,t}^\eta - q_{E,t}^\delta\|_2^2\\
    &\leq |F| \|q_{E,t}^\eta - q_{E,t}^\delta\|_2^2
\end{align*}
where
\begin{itemize}
    \item the first inequality holds by the Cauchy-Schwarz inequality,
    \item the second inequality holds because $\xi_{j,t}$ is a probability distribution over the actions set $[K]$,
    \item the last inequality holds because $|E| \leq |F|$
\end{itemize}
Therefore,
\begin{equation*}
    \|\pi_{E,t}^\eta - \pi_{E,t}^\delta\|_2 \leq \frac{|F|}{\varepsilon^2} \left( \varepsilon |\eta_n - \delta_n|(t-1) + \delta_n \sum_{s=1}^{t-1}\|\pi_{E,s}^\eta - \pi_{E,s}^\delta \|_2\right).
\end{equation*}
Using the discrete Gronwall Lemma, for all $t \leq T_\varepsilon$,
\begin{equation*}
    \|\pi_{E,t}^\eta - \pi_{E,t}^\delta\|_2 \leq \frac{|F|}{\varepsilon} |\eta_n - \delta_n|T_\varepsilon \prod_{s=1}^{t-1} \left(1+ \frac{|F|}{\varepsilon^2}\delta_n\right) \leq \frac{|F|}{\varepsilon} |\eta_n - \delta_n|T_\varepsilon \exp \left(\frac{|F|}{\varepsilon^2}\delta_n T_\varepsilon \right) .
\end{equation*}
If Assumption~\ref{hyp_expert_model} is satisfied, then since  $\delta_n \leq R_n = \frac{R}{\sqrt{n}}$ and $T_\varepsilon \leq \left(\frac{1}{|F|}-\frac{\varepsilon}{\rho}\right)\frac{\sqrt{n}}{R}$,
    \begin{equation*}
        \|\pi_{E,t}^\eta - \pi_{E,t}^\delta\|_2 \leq \frac{|F|}{R \varepsilon}\left(\frac{1}{|F|}-\frac{\varepsilon}{\rho}\right) \exp\left(\frac{|F|}{\varepsilon^2} \left(\frac{1}{|F|}-\frac{\varepsilon}{\rho}\right)\right)|\eta-\delta|.
    \end{equation*}
To conclude note that $x \to \ln(x)$ is $1/\varepsilon$-Lipschitz on $[\varepsilon,+\infty)$.
\end{proof}

\end{document}